\documentclass{article}

\usepackage{times}
\usepackage{graphicx} 
\usepackage{subfigure} 
\usepackage{natbib}
\usepackage{algorithm}
\usepackage{algorithmic}
\usepackage{hyperref}

\usepackage[accepted]{icml2016}
\usepackage{amssymb, amsmath, mathrsfs, amsthm, stmaryrd, upgreek, mathtools}
\usepackage{pgfplotstable}
\usepackage{booktabs}
\usetikzlibrary{calc,shadows}
\usepgfplotslibrary{groupplots}
\usepackage{amsmath}
\usepackage{amsfonts}
\usepackage{amssymb,pgfplots}
\usepackage{amsbsy}
\usepackage{graphicx}
\usepackage{default}
\usepackage{color}
\usepackage{enumerate}
\usepackage{tikz}
\usepackage{amsthm}
\usepackage{dsfont}
\usepackage{url}
\usepackage{multirow}

\newcommand{\C}{\mathbb{C}}

\newcommand{\Z}{\mathbb{Z}}
\newcommand{\R}{\mathbb{R}}
\newcommand{\Real}{\text{Re}}
\newcommand{\Imag}{\text{Im}}
\newcommand{\la}{\lambda}
\newtheorem{theorem}{Theorem}

\newtheorem{lemma}{Lemma}
\newtheorem{definition}{Definition}
\newtheorem{proposition}{Proposition}

\newtheorem{remark}{Remark}

\icmltitlerunning{Discrete Deep Feature Extraction: A Theory and New Architectures}

\begin{document} 

\twocolumn[
\icmltitle{Discrete Deep Feature Extraction: A Theory and New Architectures}

\icmlauthor{Thomas Wiatowski$^1$}{withomas@nari.ee.ethz.ch}
\icmlauthor{Michael Tschannen$^1$}{michaelt@nari.ee.ethz.ch}
\icmlauthor{Aleksandar Stani\'c$^1$}{astanic@student.ethz.ch}
\icmlauthor{Philipp Grohs$^2$}{philipp.grohs@univie.ac.at}
\icmlauthor{Helmut B\"olcskei$^1$}{boelcskei@nari.ee.ethz.ch}
\icmladdress{$^1$Dept. IT \& EE, ETH Zurich, Switzerland}
            \vspace{-0.3cm}
\icmladdress{$^2$Dept. Math., University of Vienna, Austria}

\icmlkeywords{Deep convolutional neural networks, scattering networks, frame theory, feature extraction, signal classification.}

\vskip 0.3in
]
\begin{abstract}
First steps towards a mathematical theory of deep convolutional neural networks for feature extraction were made---for the continuous-time case---in Mallat, 2012, and Wiatowski and B\"olcskei, 2015. This paper considers the discrete case, introduces  new convolutional neural network architectures, and proposes a mathe\-matical framework for their analysis. Specifi\-cally, we establish deformation and translation sensitivity results of local and global nature, and we investigate how certain structural properties of the input signal are reflected in the corresponding feature vectors. Our theory applies to general filters and general Lipschitz-continuous non-linearities and pooling operators. Experiments on handwritten digit classification and facial landmark detection---including feature importance evaluation---complement the theoretical findings. 
\end{abstract}

\section{Introduction}
Deep convolutional neural networks (DCNNs) have proven tremendously successful in a wide range of machine lear\-ning tasks \cite{Bengio,Nature}. Such networks are composed of multiple layers, each of which computes convolutional transforms followed by the application of non-linearities and pooling operators. 

DCNNs are typically distinguished according to (i) whether the filters employed are learned (in a supervised \cite{LeCunProc,Huang,Jarrett} or unsupervised \cite{Poultney,Ranzato,Jarrett} fashion) or pre-specified (and structured, such as, e.g., wavelets \cite{Serre,GaborLowe,MallatS}, or unstructured, such as random filters \cite{Ranzato,Jarrett}), (ii) the non-linearities  used (e.g., logistic sigmoid, hyperbolic tangent, modulus, or rectified linear unit), and (iii) the pooling operator employed (e.g., sub-sampling, average pooling, or max-pooling). While a given choice of filters, non-linearities, and pooling operators will lead to vastly diffe\-rent performance results across datasets, it is remarkable that the overall DCNN architecture allows for impressive classification results across an extraordinarily broad range of applications. It is therefore of significant interest to understand the mechanisms underlying this universality. 

First steps towards addressing this question and develo\-ping a mathematical theory of DCNNs for feature extraction were made---for the continuous-time case---in \cite{MallatS,wiatowski2015mathematical}. Specifically, \cite{MallatS} analyzed so-called scattering networks, where  signals are propagated through layers that employ directional wavelet filters and modulus non-linearities but no intra-layer pooling. The resulting wavelet-modulus feature extractor is horizontally (i.e., in every network layer) translation-invariant (accomplished by letting the wavelet scale parameter go to infinity) and deformation-stable, both properties of significance in practical feature extraction applications. Recently, \cite{wiatowski2015mathematical} considered Mallat-type networks with arbitrary filters (that may be learned or pre-specified), general Lipschitz-continuous non-linearities (e.g., rectified linear unit, shifted logistic sigmoid, hyperbolic tangent, and the modulus function), and a continuous-time pooling operator that amounts to a dilation. The essence of the results in \cite{wiatowski2015mathematical} is that vertical (i.e., asymptotically in the network depth) translation invariance and Lipschitz continuity of the feature extractor are induced by the network structure per se rather than the specific choice of filters and non-linearities. For band-limited signals \cite{wiatowski2015mathematical}, cartoon functions \cite{grohs_wiatowski}, and Lipschitz-continuous functions \cite{grohs_wiatowski}, Lip\-schitz continuity of the feature extractor automatically leads to bounds on deformation sensitivity. 

A discrete-time setup for wavelet-modulus scatte\-ring networks (referred to as ScatNets) was considered in \cite{Bruna}. 

\paragraph{Contributions.} The purpose of the present paper is to develop a theory of discrete DCNNs for feature extraction. Specifically, we follow the philosophy put forward in \cite{wiatowski2015mathematical,grohs_wiatowski}. Our theory incorporates gene\-ral filters, Lipschitz non-linearities, and Lipschitz pooling operators. In addition, we introduce and analyze a wide variety of new network architectures which build the feature vector from subsets of the layers. This leads us to the notions of local and global feature vector properties with   globa\-lity pertaining to characteristics brought out by the union of features across all network layers, and locality identifying attributes made explicit in individual layers.

Besides providing analytical performance results of gene\-ral validity, we also investigate how certain structural pro\-perties of the input signal are reflected in the corres\-ponding feature vectors. Specifically, we analyze the (local and global) deformation and translation sensitivity properties of feature vectors corresponding to sampled cartoon functions \cite{Cartoon}. For simpli\-city of exposition, we focus on the $1$-D case throughout the paper, noting that the extension to the higher-dimensional case does not pose any significant difficulties. 

Our theoretical results are complemented by extensive numerical studies on facial landmark detection and handwritten digit classification. Specifically, we elucidate the role of local feature vector properties through a feature relevance study. 

\paragraph{Notation.}
The complex conjugate of $z \in \mathbb{C}$ is denoted by $\overline{z}$. We write $\Real(z)$ for the real, and $\Imag(z)$ for the ima\-ginary part of $z \in \mathbb{C}$. We let $H_N:=\{ f:\mathbb{Z} \to \C \ | \ f[n]=f[n+N], \ \forall\, n \in \mathbb{Z}\}$ be the set of $N$-periodic discrete-time signals\footnote{We note that $H_N$ is isometrically isomorphic to $\C^N$, but we prefer to work with $H_N$ for the sake of expositional simplicity.}, and set $I_N:=\{ 0,1,\dots, N-1\}$. The delta function $\delta \in H_N$ is $\delta[n]:=1$, for $n=kN$, $k\in \mathbb{Z}$, and $\delta[n]:=0$, else. For $f,g \in H_N$, we set $\langle f,g \rangle :=\sum_{k \in I_N} f[k]\overline{g[k]}$, $\| f\|_1:=\sum_{n\in I_N}|f[n]|$, $\| f \|_2:=(\sum_{n\in I_N}|f[n]|^2)^{1/2}$, and $\| f\|_\infty:=\sup_{n\in I_N}|f[n]|$. We denote the discrete Fourier transform (DFT) of $f \in H_N$ by $\hat{f}[k]:=\sum_{n\in I_N}f[n]e^{-2\pi i kn/N }$. The circular convolution of $f\in H_N$ and $g\in H_N$ is $(f\ast g)[n]:=\sum_{k\in I_N}f[k]g[n-k]$. We write $(T_mf)[n]:=f[n-m]$, $m \in \Z$, for the cyclic translation operator. The supremum norm of a continuous-time function $c:\R \to \C$ is $\| c \|_\infty:=\sup_{x\in \R }|c(x)|$. The indicator function of an interval $[a,b]\subseteq \R$ is defined as $\mathds{1}_{[a,b]}(x):=1$, for $x\in [a,b]$, and $\mathds{1}_{[a,b]}(x):=0$, for $x\in \R\backslash{[a,b]}$. The cardinality of the set $\mathcal{A}$ is denoted by $\text{card}(\mathcal{A})$.

\section{The basic building block}\label{sec:frametheory}
The basic building block of a DCNN, described in this section, consists of a convolutional transform followed by a non-linearity and a pooling operation. 

\vspace{-0.2cm}
\subsection{Convolutional transform}\label{sec:filters}
A convolutional transform is made up of a set of filters $\Psi_\Lambda=\{ g_\la \}_{\la \in \Lambda}$. The  finite index set $ \Lambda$ can be thought of as labeling a collection of scales, directions, or frequency-shifts. The filters $g_\lambda$---referred to as atoms---may be learned (in a supervised or unsupervised fashion), pre-specified and unstructured such as random filters, or pre-specified and structured such as wavelets, curvelets, shearlets, or Weyl-Heisenberg functions.

\begin{definition}\label{defn:contframe}
Let $\Lambda$ be a finite index set. The collection $\Psi_\Lambda=\{ g_\la \}_{\la \in \Lambda}\subseteq H_N$ is called a convolutional set with Bessel bound $B\geq 0$ if 
\begin{equation}\label{condii}
\sum_{\lambda\in \Lambda} \| f \ast g_{\lambda} \|_2^2 \leq B \| f \|_2^2, \hspace{0.5cm}\forall f \in H_N.
\end{equation}
\end{definition}

\vspace{-0.3cm}
Condition \eqref{condii} is equivalent to
\begin{equation}\label{freqcover}
 \sum_{\lambda \in \Lambda} |\widehat{{g_{\lambda}}}[k]|^2\leq B, \hspace{0.5cm} \forall k \in I_N,\vspace{-0.2cm}
\end{equation}
and hence, every finite set $\{ g_\la \}_{\la \in \Lambda}$ is a convolutional set with Bessel bound  
 $B^\ast:=\max_{k\in I_N} \sum_{\lambda \in \Lambda} |\widehat{{g_{\lambda}}}[k]|^2$. As $(f\ast g_\lambda)[n]=\big\langle f,\overline{g_\lambda [n-\cdot]}\big\rangle$, $n\in I_N$, $\lambda \in \Lambda$, the outputs of the filters $g_\lambda$ may be interpreted as inner products of the input signal $f$ with translates of the atoms $g_\lambda$. Frame theory \cite{Daubechies} therefore tells us that the existence of a lower bound $A>0$ in \eqref{freqcover} according to 
 \begin{equation}\label{eq:frame1}
 A\leq \sum_{\lambda \in \Lambda} |\widehat{{g_{\lambda}}}[k]|^2\leq B, \hspace{0.5cm} \forall k \in I_N, \vspace{-0.2cm} 
 \end{equation}implies that every element  in $H_N$ can be written as a linear combination of elements in the set $\big\{\overline{g_\lambda[n-\cdot]}\big\}_{n\in I_N, \lambda\in \Lambda}$ (or in more technical parlance, the set $\big\{\overline{g_\lambda[n-\cdot]}\big\}_{n\in I_N, \lambda\in \Lambda}$ is complete for $H_N$). The absence of a lower bound $A> 0$ may therefore result in $\Psi_\Lambda$ failing to extract  essential features of the signal $f$. We note, however, that even learned filters are like\-ly to sa\-tisfy \eqref{eq:frame1} as all that is needed is, for each $k\in I_N$, to have $\widehat{{g_{\lambda}}}[k]\neq 0$ for at least one $\lambda \in \Lambda$.  As we shall see below, the existence of a lower bound $A>0$ in \eqref{eq:frame1} is, however, not needed for our theory to apply. 

Examples of structured convolutional sets with $A=B=1$ include, in the $1$-D case, wavelets \cite{Daubechies} and Weyl-Heisenberg functions \cite{bolcskei1997discrete}, and in the $2$-D case, tensorized wavelets \cite{MallatW}, curvelets \cite{candes2006fast}, and shearlets \cite{Shearlets}.

\subsection{Non-linearities}\label{sec:exmpnon}
The non-linearities $\rho:\mathbb{C}\to \mathbb{C}$ we consider are all point-wise and satisfy the Lipschitz property $|\rho(x) - \rho(y)| \leq L | x-y|$, $\forall x,y\in \C$, for some $L>0$. 

\subsubsection{Example non-linearities} \label{sec:exmpnonlist}
\begin{itemize}
\item{The \textit{hyperbolic tangent} non-linearity, defined as $\rho(x)=\tanh(\Real(x))+i\tanh(\Imag(x))$, where $\tanh(x)=\frac{e^{x}-e^{-x}}{e^{x}+e^{-x}}$, has Lipschitz constant $L=2$.}
\item{The \textit{rectified linear unit} non-linearity is given by $\rho(x)=\max\{0,\Real(x)\}+i\max\{0,\Imag(x)\}$, and has Lipschitz constant $L=2$.}
\item{The \textit{modulus} non-linearity is $\rho(x)=|x|$, and has Lip\-schitz constant $L=1$.}
\item{The \text{logistic sigmoid} non-linearity is defined as $\rho(x)=\text{sig}(\Real(x))+i\,\text{sig}(\Imag(x))$, where $\text{sig}(x)=\frac{1}{1+e^{-x}}$, and has Lipschitz constant $L=1/2$.}
\end{itemize}
We refer the reader to \cite{wiatowski2015mathematical} for proofs of the Lipschitz properties of these example non-linearities. 

\subsection{Pooling operators}
The essence of pooling is to reduce signal dimensionality in the individual network layers and to ensure robustness of the feature vector w.r.t. deformations and translations. 

The theory developed in this paper applies to general pooling operators $P:H_{N}\to H_{N/S}$, where $N,S \in \mathbb{N}$ with $N/S \in \mathbb{N}$, that satis\-fy the Lipschitz property $\| Pf - Pg\|_2 \leq R \|f-g \|$, $\forall f,g \in H_N$, for some $R>0$. The integer $S$ will be referred to as pooling factor, and determines the ``size'' of the neighborhood values are combined in. 

\subsubsection{Example pooling operators}\label{sec:exmppooling}
\begin{itemize}
\item{\textit{Sub-sampling}, defined as  $P:H_{N}\to H_{N/S}$, $(Pf)[n]=f[Sn],$ $n\in I_{N/S}$, has Lipschitz constant $R=1$. For $S=1$, $P$ is the  identity operator which amounts to ``no pooling''. }
\item{\textit{Averaging}, defined as $P:H_{N}\to H_{N/S}$, $(Pf)[n]=\sum_{k=Sn}^{Sn+S-1}\alpha_{k-Sn}f[k],$ $n\in I_{N/S}$, has Lip\-schitz constant $R=S^{1/2}\max_{k\in \{0,\dots, S-1 \}} |\alpha_k|$. The weights $\{\alpha_k\}_{k=0}^{S-1}$ can be learned \cite{LeCunProc} or pre-specified \cite{Pinto} (e.g., uniform pooling corresponds to $\alpha_k=\frac{1}{S}$, for $k\in \{0,\dots,S-1 \}$).}
\item{\textit{Maximization}, defined as $P:H_{N}\to H_{N/S}$, $(Pf)[n]=\max_{k\in \{Sn,\dots,Sn+S-1 \}}|f[k]|,$ $n\in I_{N/S}$, has Lipschitz constant $R=1$. }
\end{itemize}
We refer to Appendix \ref{app:proofLip} in the Supplement for proofs of the Lipschitz property of these three example pooling operators along with the derivations of the corresponding Lipschitz constants. 

\section{The network architecture}
The architecture we consider is flexible in the following sense. In each layer, we can feed into the feature vector either the signals propagated down to that layer (i.e., the feature maps), filtered versions thereof, or we can decide not to have that layer contribute to the feature vector. 

The basic building blocks of our network are the triplets $(\Psi_d,\rho_d,P_d)$ of filters, non-linearities, and pooling operators associated with the $d$-th network layer and referred to as \textit{mo\-dules}. We emphasize that these triplets are allowed to be different across layers.

\begin{definition} For network layers $d$, $1\leq d \leq D$, let $\Psi_d=\{g_{\la_d} \}_{\la_d \in \Lambda_d}\subseteq H_{N_d}$  be a  convolutional set, $\rho_d:\C \to \C$ a  point-wise Lipschitz-continuous non-linearity, and $P_d:H_{N_d} \to H_{N_{d+1}}$  a Lipschitz-continuous pooling operator with $N_{d+1}=\frac{N_d}{S_d}$, where $S_d\in \mathbb{N}$ denotes the pooling factor in the $d$-th layer. Then, the sequence of triplets
\vspace{-0.1cm}
$$
\Omega:=\Big( (\Psi_d,\rho_d,P_d)\Big)_{1\leq d \leq D}\vspace{-0.1cm}
$$
is called a module-sequence. 
\end{definition}
Note that the dimensions of the spaces $H_{N_d}$ satisfy $N_1 \geq N_2 \geq \ldots \geq N_D$. Associated with the module $(\Psi_d,\rho_d,P_d)$, we define the operator 
\begin{equation}\label{eq:e1} 
(U_d[\lambda_d]f):=P_d(\rho_d(f \ast g_{\lambda_d}))
\end{equation}
and extend it to paths on index sets $$q=(\lambda_1,\lambda_2,\dots,\lambda_d)\in \Lambda_1\times \Lambda_2\times \dots \times \Lambda_d:=\Lambda_{1}^d,$$ for $1\leq d \leq D$, according to
\begin{equation}\label{aaaaa}
\begin{split}
U[q]f=&\,U[(\lambda_1,\lambda_{2},\dots,\lambda_d)]f\\:=&\, U_d[\lambda_d] \cdots U_{2}[\lambda_{2}]U_{1}[\lambda_{1}]f.
\end{split}
\end{equation}
For the empty path $e:=\emptyset$ we set $\Lambda_1^0:=\{ e \}$ and let $U[e]f:=f$, for all $f\in H_{N_1}$. 
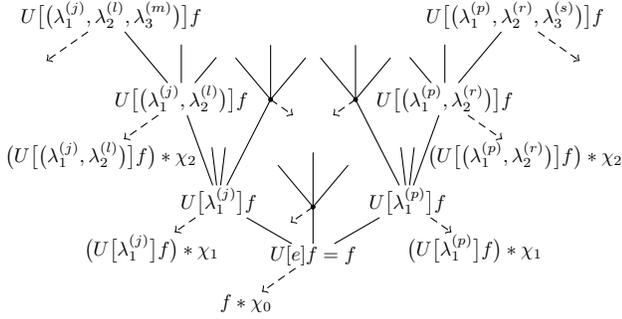
\begin{figure}[t!]
\centering
\begin{tikzpicture}[scale=1.45,level distance=10mm]

  \tikzstyle{every node}=[rectangle, inner sep=1pt, scale=0.75]
  \tikzstyle{level 1}=[sibling distance=30mm]
  \tikzstyle{level 2}=[sibling distance=10mm]
  \tikzstyle{level 3}=[sibling distance=4mm]
  \node {$U[e]f=f$}
	child[grow=90, level distance=.45cm] {[fill=gray!50!black] circle (0.5pt)
		child[grow=130,level distance=0.5cm] 
        		child[grow=90,level distance=0.5cm] 
        		child[grow=50,level distance=0.5cm]
		child[level distance=.25cm,grow=215, densely dashed, ->] {}  
	}
        child[grow=150] {node {$U\big[\lambda_1^{(j)}\big]f$}
	child[level distance=.75cm,grow=215, densely dashed, ->] {node {$\big(U\big[\lambda_1^{(j)}\big]f\big)\ast\chi_{1}$}
	}
	child[grow=83, level distance=0.5cm] 
	child[grow=97, level distance=0.5cm] 
        child[grow=110] {node {$U\big[\big(\lambda_1^{(j)},\lambda_2^{(l)}\big)\big]f$}
	child[level distance=.9cm,grow=215, densely dashed, ->] {node {$\big(U\big[\big(\lambda_1^{(j)},\lambda_2^{(l)}\big)\big]f\big)\ast\chi_{2}$}
	}
        child[grow=130] {node {$U\big[\big(\lambda_1^{(j)},\lambda_2^{(l)},\lambda_3^{(m)}\big)\big]f$}
	child[level distance=0.75cm,grow=215, densely dashed, ->] {node {}}
	}
        child[grow=90,level distance=0.5cm]
 	child[grow=50,level distance=0.5cm]
       }
       child[grow=63, level distance=1.05cm] {[fill=gray!50!black] circle (0.5pt)
	child[grow=130,level distance=0.5cm] 
       child[grow=90,level distance=0.5cm] 
       child[grow=50,level distance=0.5cm]   
       child[level distance=.25cm,grow=325, densely dashed, ->] {}    
       }
       }
       child[grow=30] {node {$U\big[\lambda_1^{(p)}\big]f$}
       child[level distance=0.75cm, grow=325, densely dashed, ->] {node {$\big(U\big[\lambda_1^{(p)}\big]f\big)\ast\chi_{1}$}
	}
	child[grow=83, level distance=0.5cm] 
	child[grow=97, level distance=0.5cm]
        child[grow=117, level distance=1.05cm] {[fill=gray!50!black] circle (0.5pt)
        child[grow=130,level distance=0.5cm] 
        child[grow=90,level distance=0.5cm]
        child[grow=50,level distance=0.5cm] 
        child[level distance=.25cm,grow=215, densely dashed, ->] {}  
	 }
        child[grow=70] {node {$U\big[\big(\lambda_1^{(p)},\lambda_2^{(r)}\big)\big]f$}
	 child[level distance=0.9cm,grow=325, densely dashed, ->] {node {$\big(U\big[\big(\lambda_1^{(p)},\lambda_2^{(r)}\big)\big]f\big)\ast\chi_{2}$}}
	child[grow=130,level distance=0.5cm] 
         child[grow=90,level distance=0.5cm] 
                      child[grow=50] {node {$U\big[\big(\lambda_1^{(p)},\lambda_2^{(r)},\lambda_3^{(s)}\big)\big]f$}
             child[level distance=0.75cm,grow=325, densely dashed, ->] {node {}}}
	}
     }
	child[level distance=0.75cm, grow=215, densely dashed, ->] {node {$f\ast \chi_0$}};
\end{tikzpicture}
\caption{Network architecture  underlying the feature extractor  \eqref{ST}. The index $\lambda_{d}^{(k)}$ corresponds to the $k$-th atom $g_{\lambda_{d}^{(k)}}$ of the convolutional set $\Psi_d$ associated with the $d$-th network layer. The function $\chi_{d}$ is the output-generating atom of the $d$-th layer. The root of the network corresponds to $d=0$.} 
\label{fig:gsn}
\end{figure}

The network output in the $d$-th layer is given by $(U[q]f)\ast \chi_d$, $q\in \Lambda_1^d$, where $\chi_d \in H_{N_{d+1}}$ is referred to as output-generating atom. Specifically, we let $\chi_d$ be (i) the delta function $\delta[n]$, $n \in I_{N_{d+1}}$, if we want the output to equal the unfiltered features $U[q]f$, $q\in \Lambda_1^d$, propagated down to layer $d$, or (ii) any other signal of length $N_{d+1}$, or (iii) $\chi_d=0$ if we do not want layer $d$ to contribute to the feature vector. From now on we formally add $\chi_d$ to the set $\Psi_{d+1}= \{ g_{\lambda_{d+1}} \}_{\lambda_{d+1}\in \Lambda_{d+1}}$, noting that $\{ g_{\lambda_{d+1}}\}_{\lambda_{d+1} \in \Lambda_{d+1}}\cup \{ \chi_{d}\}$ forms a convolutional set $\Psi'_{d+1}$ with Bessel bound $B'_{d+1} \leq B_{d+1} + \max_{k\in I_{N_{d+1}}} | \widehat \chi_d[k]|^2$. We emphasize that the atoms of the augmented set $\{ g_{\lambda_{d+1}}\}_{\lambda_{d+1} \in \Lambda_{d+1}}\cup \{ \chi_{d}\}$ are employed across two consecutive layers in the sense of $\chi_{d}$ genera\-ting the output in the $d$-th layer accor\-ding to $(U[q]f)\ast \chi_d$, $q\in \Lambda_{1}^{d}$, and the remaining atoms $\{ g_{\lambda_{d+1}}\}_{\lambda_{d+1} \in \Lambda_{d+1}}$  propa\-gating the signals $U[q]f$, $q\in \Lambda_{1}^{d}$, from the $d$-th layer down to the $(d+1)$-st layer accor\-ding to \eqref{eq:e1}, see Fig. \ref{fig:gsn}. With slight abuse of notation, we shall henceforth write $\Psi_d$ for $\Psi'_d$ and $B_d$ for $B'_d$ as well.

We are now ready to define the feature extractor $\Phi_\Omega$ based on the module-sequence $\Omega$. 
\begin{definition}\label{defn2}
 Let $\Omega=\big( (\Psi_d,\rho_d,P_d)\big)_{1\leq d \leq D}$ be a module-sequence. The feature extractor $\Phi_\Omega$ based on $\Omega$ maps $f\in H_{N_1}$ to its features 
\vspace{-0.3cm}
\begin{equation}\label{ST}
\Phi_\Omega(f):=\bigcup_{d=0}^{D-1} \Phi^d_\Omega(f),\vspace{-0.2cm}
\end{equation}
where $\Phi^d_\Omega(f):=\{ (U[q]f)\ast \chi_d \}_{q\in \Lambda_1^d}$ is the collection of features generated in the $d$-th network layer (see Fig. \ref{fig:gsn}).
\end{definition}
The dimension of the feature vector $\Phi_\Omega(f)$ is given by $\varepsilon_0N_{1}+\sum_{d=1}^{D-1}\varepsilon_dN_{d+1}\big(\prod_{k=1}^d\text{card}(\Lambda_k)\big)$, where $\varepsilon_d=1$, if an output is generated (either filtered or unfiltered) in the $d$-th network layer, and $\varepsilon_d=0$, else. As $N_{d+1}=\frac{N_d}{S_d}=\dots=\frac{N_1}{S_1\cdots S_d}$, for $d\geq 1$, the dimension of the overall feature vector is determined by the pooling factors $S_k$ and, of course, the layers that contribute to the feature vector.
\begin{remark}
It was argued in \cite{Bruna,Anden,Oyallon} that the features $\Phi^1_\Omega(f)$ when generated by wavelet filters, modu\-lus non-linearities, without intra-layer pooling, and by employing output-generating atoms with low-pass characte\-ristics, describe mel frequency cepstral coefficients \cite{Davis} in $1$-D, and SIFT-descriptors \cite{Lowe,Tola} in $2$-D. 
\end{remark}

\section{Sampled cartoon functions}\label{sec-cartoon}
While our main results hold for general signals $f$, we can provide a refined analysis for the class of sampled cartoon functions. This allows to understand how certain structural properties of the input signal, such as the pre\-sence of sharp edges, are reflected in the feature vector. Cartoon functions---as introduced in continuous time in \cite{Cartoon}---are piecewise ``smooth''  apart from curved discontinuities along Lipschitz-continuous hypersurfaces. They hence provide a good model for natural images (see Fig. \ref{fig:data}, left) such as those in the Caltech-256 \cite{Caltech256} and the CIFAR-100 \cite{CIFAR2} datasets, for ima\-ges of handwritten digits \cite{MNIST} (see Fig. 2, middle), and for images of geometric objects of different shapes, sizes, and colors as in the Baby AI School dataset\footnote{\url{http://www.iro.umontreal.ca/\%7Elisa/twiki/bin/view.cgi/Public/BabyAISchool}}.

Bounds on deformation sensitivity for cartoon functions in continuous-time DCNNs were recently reported in \cite{grohs_wiatowski}. Here, we analyze deformation sensitivity for  sampled cartoon functions passed through discrete DCNNs. 
\begin{definition}\label{def:1d-cartoon}
The function $c:\R \to \C$ is referred to as a cartoon function if it can be written as $c = c_1 + \mathds{1}_{[a,b]}c_2$, where $[a,b]\subseteq [0,1]$ is a closed interval, and $c_i:\mathbb{R}\to \mathbb{C}$, $i=1,2$, satisfies the Lipschitz property
\begin{align*}\label{eq:decay}
|c_i(x)-c_i(y)|\leq C|x-y|, \quad \forall x,y\in \mathbb{R},
\end{align*}
for some $C>0$. Furthermore, we denote by 
\begin{align*}
\mathcal{C}^{K}_{\mathrm{CART}}:=\{&c_1 + \mathds{1}_{[a,b]}c_2 \ | \ |c_i(x)-c_i(y)|\leq K|x-y|, \\
& \forall\, x, y\in \mathbb{R},\  i=1,2, \ \|c_2\|_\infty \leq K\}
\end{align*} 
the class of cartoon functions of variation $K>0$, and by
\begin{align*}
\mathcal{C}^{N,K}_{\mathrm{CART}}:=&\Big\{f[n]=c(n/N), \ n\in \{0,1,\dots, N-1\} \ \Big| \\
&c=(c_1 + \mathds{1}_{[a,b]}c_2) \in \mathcal{C}^{K}_{\mathrm{CART}} \ \text{with} \\
 & \ a,b \notin \Big\{0,\frac{1}{N}, \dots, \frac{N-1}{N}\Big\} \Big\}
\end{align*}
the class of sampled cartoon functions of length $N$ and variation $K>0$.
\end{definition}

\newcommand{\imgsize}{.14\textwidth}
\newcommand{\pltsize}{.23\textwidth}
\newcommand{\liney}{0.0895\textwidth}
\newcommand{\linemargin}{0.01\textwidth}
\begin{figure}
\centering
	\includegraphics[width = .15\textwidth]{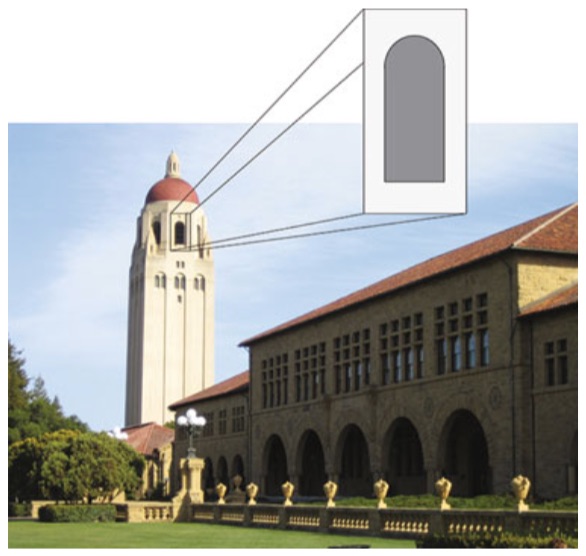}	
	\begin{tikzpicture}[scale=1]
  	\node[inner sep=0pt] (img) at (\imgsize,0)
    	{\includegraphics[width=\imgsize]{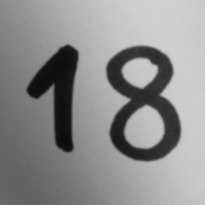}};
   	\draw[-,dashed] ($(img.north west) + (-\linemargin,-\liney)$) -- ($(img.north east) + (\linemargin,-\liney)$);
	\end{tikzpicture}
	\begin{tikzpicture}[scale=1]
	\begin{axis}[width = \pltsize,axis lines=center, ticks=none]
    		\addplot +[mark=none,solid,black] table[x index=0,y index=1]{CartoonCorsssection.dat};
    	\end{axis}
	\end{tikzpicture}
	
	\caption{Left: A natural image (image credit: \cite{ShearletsIntro}) is typically governed by  are\-as of little variation, with the individual areas separated by edges that can be modeled as  curved singularities. Middle: Image of a handwritten digit. Right: Pixel values corresponding to the dashed row in the middle image.}
	\label{fig:data}
	\vspace{-0.1cm}
\end{figure}
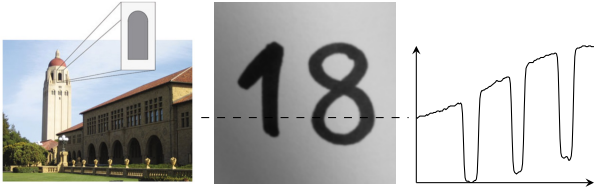

We note that excluding the boundary points $a,b$ of the interval $[a,b]$ from being sampling points $n/N$ in the definition of $\mathcal{C}^{N,K}_{\mathrm{CART}}$ is of conceptual importance (see Remark~\ref{rem:def} in the Supplement). Moreover, our results can easily be generalized to classes  $\mathcal{C}^{N,K}_{\mathrm{CART}}$ consis\-ting of functions $f[n]=c(n/N)$ with $c$ containing multiple  ``1-D edges'' (i.e., multiple discontinuity points) according to $c=c_1+\sum_{l=1}^{L}\mathds{1}_{[a_l,b_l]}c_2$ with $\cap_{l=1}^L[a_l,b_l]=\emptyset$. We also note that $\mathcal{C}^{N,K}_{\mathrm{CART}}$ reduces to the class of sampled Lipschitz-continuous functions upon setting $c_2=0$.

 A sampled cartoon function in $2$-D models, e.g., an image\- acquired by a digital camera (see Fig. \ref{fig:data}, middle); in $1$-D, $f \in \mathcal{C}^{N,K}_{\mathrm{CART}}$ can be thought of as the pixels in a row or column of this image (see Fig. \ref{fig:data} right, which shows a cartoon function with $6$ discontinuity points).

\section{Analytical results}\label{sec:PER} 
We analyze global and local feature vector properties with globality pertaining to characteristics brought out by the union of features across all network layers, and locality identifying attributes made explicit in individual layers.

\subsection{Global properties}\label{sec:defstab}
\begin{theorem}\label{mainmain}
 Let $\Omega=\big( (\Psi_d, \rho_d, P_d)\big)_{1\leq d \leq D}$ be a module-sequence. Assume that the Bessel bounds $B_d>0$, the Lipschitz constants $L_d>0$ of the non-linearities $\rho_d$, and the Lipschitz constants $R_d>0$ of the pooling operators $P_d$ satisfy 
 \vspace{-0.2cm}
\begin{equation}\label{weak_admiss2}
 \max_{1\leq d \leq D}\max\{B_d,B_dR_d^2L_d^2 \}\leq 1. \vspace{-0.2cm}
 \end{equation} 
\begin{itemize}
\item[i)]{The feature extractor $\Phi_\Omega$  is Lipschitz-continuous with Lipschitz constant $L_\Omega=1$, i.e.,   \vspace{-0.1cm}
\begin{equation}\label{eq:thm1}
||| \Phi_\Omega(f) -\Phi_\Omega(h)||| \leq \| f-h \|_2,  \vspace{-0.05cm}
\end{equation}
for  all $f,h \in H_{N_1}$, where the feature space norm is defined as  \vspace{-0.2cm}
\begin{equation}\label{eq:featspacenorm}
||| \Phi_\Omega(f)|||^2:=\sum_{d=0}^{D-1}\sum_{q\in \Lambda_{1}^d} ||(U[q]f)\ast\chi_d ||^2_2. \vspace{-0.2cm}
\end{equation}
}
\item[ii)]{If, in addition to \eqref{weak_admiss2}, for all $d\in \{1,\dots, D-1\}$ the non-linearities $\rho_d$ and the pooling operators $P_d$ sa\-tisfy $\rho_d(0)=0$ and $P_d(0)=0$ (as all non-linearities and pooling operators in Sections \ref{sec:exmpnonlist} and \ref{sec:exmppooling}, apart from the logistic sigmoid non-linearity, do), then
\vspace{-0.25cm}
\begin{equation}\label{eq:thm2}
||| \Phi_\Omega(f)||| \leq \| f \|_2, \hspace{0.5cm} \forall f \in H_{N_1}.
\end{equation}
\vspace{-.75cm}}
\item[iii)]{For every variation $K>0$ and deformation $F_\tau$ of the form
\vspace{-0.15cm}
\begin{align}\label{eq:def0}
\hspace{-0.4cm}(F_\tau f)[n]:&=c(n/N_1 -\tau (n/N_1)),\quad n\in I_{N_1},
\end{align}
where $\tau:\mathbb{R}\to [-1,1]$, the deformation sensitivity is bounded according to 
\vspace{-0.15cm}
\begin{align}
&||| \Phi_\Omega(F_{\tau} f)-\Phi_\Omega(f) |||\leq 4K N_1^{1/2}\| \tau \|^{1/2}_\infty,\label{mainmainmain}\vspace{-0.25cm}
\end{align}
\vspace{-0.25cm}
for all $f\in \mathcal{C}^{N_1,K}_{\mathrm{CART}}$.
}
\end{itemize}
\end{theorem}
\begin{proof}
See Appendix \ref{proof:defo} in the Supplement.
\end{proof}
\vspace{-0.25cm}
The Lipschitz continuity \eqref{eq:thm1}  gua\-rantees that pairwise
distances of input signals do not increase through feature extraction. As an immediate implication of the Lipschitz continuity we get robustness of the feature extractor w.r.t. additive bounded noise $\eta \in H_{N_1}$ in the sense of 
\vspace{-0.15cm}
\begin{equation*}\label{eq:thm5}
||| \Phi_\Omega(f+\eta) -\Phi_\Omega(f)||| \leq \| \eta \|_2, \vspace{-0.15cm}
\end{equation*}
for  all $f \in H_{N_1}$. 
\begin{remark}
As detailed in the proof of Theorem \ref{mainmain}, the Lipschitz continuity \eqref{eq:thm1} combined with the deformation sensitivity bound (see Proposition \ref{prop:main}
in the Supplement) for the signal class under consideration, namely sampled cartoon functions, 
establishes the deformation sensitivity bound \eqref{mainmainmain} for the feature extractor. This insight has important practical ramifications as it shows that 
whenever we have deformation sensitivity bounds for a signal class, we automatically get deformation sensitivity guarantees for the corresponding feature extractor.
\end{remark}

 From \eqref{mainmainmain} we can deduce a statement on the sensitivity of $\Phi_\Omega$ w.r.t. translations on $\R$. To this end, we first note that setting $\tau_t(x)=t$, $x\in \mathbb{R}$, for $t\in [-1,1]$, \eqref{eq:def0} becomes  \vspace{-0.1cm}
\begin{align*}
(F_{\tau_t} f)[n]=c(n/N_1-t),\quad \ n\in I_{N_1}\label{absc}.
\end{align*}
Particularizing \eqref{mainmainmain} accordingly, we obtain 
\begin{equation}\label{thm:eq7}
||| \Phi_\Omega(F_{\tau_t} f)-\Phi_\Omega(f) |||\leq 4K N_1^{1/2}|t|^{1/2},
\end{equation} 
which shows that small translations $|t|$ of the underlying analog signal $c(x)$, $x\in \mathbb{R}$, lead to small changes in the feature vector obtained by passing the resulting sampled signal through a discrete DCNN. We shall say that \eqref{thm:eq7} is a translation sensitivity bound. Analyzing the impact of deformations and translations over $\R$ on the discrete feature vector generated by the sampled analog signal closely models real-world phenomena (e.g., the jittered acquisition of an analog signal with a digital camera, where different values of $N_1$ in \eqref{eq:def0} correspond to different camera re\-solutions). 

We note that, while iii) in Theorem \ref{mainmain} is specific to cartoon functions, i) and ii) apply to all signals in $H_{N_1}$.

The strength of the results in Theorem \ref{mainmain} derives itself from the fact that condition \eqref{weak_admiss2} on the underlying module-sequence $\Omega$ is easily met in practice. To see this, we first note that $B_d$ is determined by the convolutional set $\Psi_d$, $L_d$ by the non-linearity $\rho_d$, and $R_d$ by the pooling operator $P_d$. Condition \eqref{weak_admiss2} is met if
\begin{align}\label{eq:thm3}
B_d\leq \min \{1,R_d^{-2}L_d^{-2} \}, \quad \forall \, d\in \{1,2,\dots, D\},\vspace{-0.25cm}
\end{align}
which, if not satisfied by default, can be enforced simply by normalizing the elements in $\Psi_d$. Specifically,  for $C_d:=\max\{B_d,R_d^2L_d^2 \}$ the set $\widetilde{\Psi}_d:=\{ C_d^{-1/2} g_{\la_d} \}_{\la_d \in \Lambda_d}$  has Bessel bound  $\widetilde{B_d}=\frac{B_d}{C_d}$ and hence satisfies \eqref{eq:thm3}. While this normalization does not have an impact on the results in Theorem \ref{mainmain}, there exists, however, a tradeoff between ener\-gy preservation and deformation (respectively translation) sensitivity in $\Phi_\Omega^d$ as detailed in  the next section.

\subsection{Local properties}\label{sec:deepinv}
\begin{theorem}\label{thm:deep}
 Let $\Omega=\big( (\Psi_d, \rho_d, P_d)\big)_{1\leq d \leq D}$ be a module-sequence with corresponding Bessel bounds $B_d>0$, Lip\-schitz constants $L_d>0$ of the non-linearities $\rho_d$, Lipschitz constants $R_d>0$ of the pooling operators $P_d$, and output-generating atoms $\chi_d$. Let further $L_\Omega^0:=\|\chi_0 \|_1$ and\,\footnote{We note that $\|\chi_d\|_1$ in \eqref{lipo21} can be upper-bounded (and hence substituted) by $B_{d+1}$, see Remark \ref{rem_proof} in the Supplement.}  \vspace{-0.2cm}
 \begin{equation}\label{lipo21}
 L_\Omega^d:=\| \chi_d\|_1 \Big( \prod_{k=1}^{d}B_kL_k^2R_k^2 \Big)^{1/2}, \hspace{0.5cm} d\geq 1. \vspace{-0.3cm} \end{equation}
 \begin{itemize}
 \item[i)]{The features generated in the $d$-th network layer are Lipschitz-continuous with  Lipschitz constant $L_\Omega^d$, i.e.,  
\begin{equation}\label{eq:thmb1}
||| \Phi^d_\Omega(f) -\Phi^d_\Omega(h)||| \leq L_\Omega^d \| f-h \|_2, 
\end{equation}
for  all $f,h \in H_{N_1}$, where $|||\Phi^d_\Omega(f)|||^2:=\sum_{q\in \Lambda_{1}^d} ||(U[q]f)\ast\chi_d ||^2_2.$}
\item[ii)]{\vspace{-0.19cm}If the non-linearities $\rho_k$ and the pooling operators $P_k$ satisfy $\rho_k(0)=0$ and $P_k(0)=0$, respectively, for all $k\in \{1,\dots, d\}$, then
\vspace{-0.1cm}
\begin{equation}\label{eq:thmb2}
||| \Phi^d_\Omega(f)||| \leq L^d_\Omega \| f \|_2, \hspace{0.5cm} \forall f \in H_{N_1}.
\end{equation}
}
\item[iii)]{\vspace{-0.1cm}For all $K>0$ and all $\tau:\mathbb{R}\to [-1,1]$, the features generated in the $d$-th network layer satisfy \begin{align}
&||| \Phi^d_\Omega(F_\tau f) -\Phi^d_\Omega(f)||| \leq 4 L_\Omega^d KN^{1/2}\|\tau \|_\infty^{1/2} \label{abc1a},
\end{align}
for all $f\in \mathcal{C}^{N_1,K}_{\mathrm{CART}}$, where $F_\tau f$ is defined in \eqref{eq:def0}.}
\item[iv)]{\vspace{-0.1cm}If the module-sequence employs sub-sampling, ave\-rage pooling, or max-pooling with corresponding pooling factors $S_d\in \mathbb{N}$, then
 \begin{equation}\label{eq:deepo}
 \Phi^d_\Omega(T_{m}f)=T_{\frac{m}{S_1\dots S_d}}\Phi^d_\Omega(f),
\end{equation}
 for all $f\in H_{N_1}$ and all $m\in \mathbb{Z}$ with $\frac{m}{S_1\dots S_d}\in \mathbb{Z}$. Here, $T_{m}\Phi^d_\Omega(f)$ refers to element-wise application of $T_m$, i.e., $T_{m}\Phi^d_\Omega(f):=\{T_m h \ | \ \forall h \in \Phi^d_\Omega(f)\}$.
}
\end{itemize} \end{theorem}
\begin{proof}
See Appendix \ref{proof:nonexpan3} in the Supplement.
\end{proof}
\vspace{-0.3cm}
One may be tempted to infer the global results \eqref{eq:thm1}, \eqref{eq:thm2}, and \eqref{mainmainmain} in Theorem \ref{mainmain} from the corresponding local results in Theorem \ref{thm:deep}, e.g., the energy bound in \eqref{eq:thm2} from \eqref{eq:thmb2} accor\-ding to
$
||| \Phi_\Omega(f)|||=\Big( \sum_{d=0}^{D-1} ||| \Phi_\Omega^d (f)|||^2 \Big)^{1/2}\leq \sqrt{D} \| f \|_2,$
where we employed $L^d_\Omega\leq 1$ owing to \eqref{weak_admiss2}. This would, however, lead to the ``global'' Lipschitz constant $L_\Omega=1$ in \eqref{eq:thm1}, \eqref{eq:thm2}, and \eqref{mainmainmain} to be replaced by $L_\Omega=\sqrt{D}$ and thereby render the corresponding results much weaker. 

Again, we emphasize that, while iii) in Theorem \ref{thm:deep} is specific to cartoon functions, i), ii), and iv) apply to all signals in $H_{N_1}$.

For a fixed network layer $d$, the ``local'' Lipschitz constant $L_\Omega^d$ determines the noise sensitivity of the features $\Phi^d_\Omega(f)$ according to
\begin{equation}\label{eq:thm_b5}
||| \Phi^d_\Omega(f+\eta) -\Phi^d_\Omega(f)||| \leq L_\Omega^d\| \eta \|_2, 
\end{equation}
where \eqref{eq:thm_b5} follows from \eqref{eq:thmb1}. Moreover, $L_\Omega^d$ via \eqref{abc1a} also quantifies the impact of deformations (or translations when $\tau_t(x)=t$, $x\in \mathbb{R}$, for $t\in [-1,1]$) on the feature vector. In practice, it may be desirable to have the features $\Phi_\Omega^d$ become more robust to additive noise and less deformation-sensitive (respectively, translation-sensitive) as we progress deeper into the network. Formally, this vertical sensitivity reduction can be induced by ensuring that $L_\Omega^{d+1}<L_\Omega^d$. Thanks to
$
L_\Omega^d=\frac{\|\chi_d\|_1 B_d^{1/2}L_dR_d}{\|\chi_{d-1} \|_1}L_\Omega^{d-1},
$
this can be accomplished by choosing the module-sequence such that $ \| \chi_d\|_1 B_d^{1/2} L_dR_d<\| \chi_{d-1}\|_1$. Note, however, that owing to \eqref{eq:thmb2} this will also reduce the signal energy contained in the features $\Phi^d_\Omega(f)$. We therefore have a tradeoff between deformation (respectively translation) sensitivity and energy preservation. Having control over this tradeoff through the choice of the module-sequence $\Omega$ may come in handy in practice.

For average pooling with uniform weights $\alpha^d_k=\frac{1}{S_d}$, $k= 0,\dots,S_d-1$ (noting that the corresponding  Lipschitz constant is $R_d=S_d^{-1/2}$, see Section  \ref{sec:exmppooling}), we get $L_\Omega^d=\| \chi_d\|_1 \Big( \prod_{k=1}^{d}\frac{B_kL_k^2}{S_k} \Big)^{1/2}$, which illustrates that pooling can have an impact on the sensitivity and energy properties of $\Phi_\Omega^d$. 

We finally turn to interpreting the translation covariance result \eqref{eq:deepo}. Owing to the condition $\frac{m}{S_1\dots S_d}\in \mathbb{Z}$, we get translation covariance only on the rough grid induced by the product of the pooling factors. In the absence of pooling, i.e., $S_k=1$, for $k\in \{ 1,\dots, d\}$, we obtain translation covariance w.r.t. the fine grid the input signal $f\in H_{N_1}$ lives on. 

\begin{remark}
We note that ScatNets \cite{Bruna} are translation-covariant on the rough grid induced by the factor $2^{J}$ corresponding to the coarsest wavelet scale. Our result in \eqref{eq:deepo} is hence in the spirit of \cite{Bruna} with the difference that the grid in our case is induced by the  pooling factors $S_k$.
\end{remark}

\urldef{\code}\url{http://www.nari.ee.ethz.ch/commth/research/}
\section[Experiments]{Experiments\footnote{Code available at \code}} \label{sec:Exp}
We consider the problem of handwritten digit classification and evaluate the performance of the feature extractor $\Phi_\Omega$ in combination with a support vector machine (SVM). The results we obtain are competitive with the state-of-the-art in the literature. The second line of experiments we perform assesses the importance of the features extracted by $\Phi_\Omega$ in facial landmark detection and in handwritten digit classification, using random forests (RF) for regression and classification, respectively. Our results are based on a DCNN with different non-linearities and pooling operators, and with tensorized (i.e., separable) wavelets as filters, sensitive to $3$ directions (horizontal, vertical, and diagonal). Furthermore, we generate outputs in all layers through low-pass filtering. Circular convolutions with the 1-D filters underlying the tensorized wavelets are efficiently implemented using the \textit{algorithme \`a trous} \cite{holschneider1989real}. 

To reduce the dimension of the feature vector, we compute features along frequency decreasing paths only \cite{Bruna}, i.e., for every node $U[q]f$, $q\in \Lambda_1^{d-1}$, we retain only those child nodes $U_{d}[\lambda_{d}]U[q]f=P_d\big(\rho_d ((U[q]f)\ast g_{\lambda_{d}}) \big)$ that correspond to wavelets $g_{\lambda_{d}}$ with scales larger than the maximum scale of the wavelets used to get $U[q]f$. We refer to \cite{Bruna} for a detailed justification of this approach for scattering networks. 

\subsection{Handwritten digit classification}\label{sec:digitclass}
We use the MNIST dataset of handwritten digits \cite{MNIST} which comprises 60,000 training and 10,000 test images of size $28 \times 28$. We set $D = 3$, and compare different network configurations, each defined by a single mo\-dule (i.e., we use the same filters, non-linearity, and pooling operator in all layers). Specifically, we consider Haar wavelets and reverse biorthogonal 2.2 (RBIO2.2) wavelets \cite{MallatW}, both with $J=3$ scales, the non-linearities described in Section \ref{sec:exmpnonlist}, and the pooling operators described in Section \ref{sec:exmppooling} (with $S_1 = 1$ and $S_2 = 2$).
We use a SVM with radial basis function (RBF) kernel for classification. To reduce the dimension of the feature vectors from 18,424 (or 50,176, for the configurations without pooling) down to 1000, we employ the supervised orthogonal least squares feature selection procedure described in \cite{Oyallon}. The penalty parameter of the SVM and the localization parameter of the RBF kernel are selected via 10-fold cross-validation for each combination of wavelet filter, non-linearity, and pooling operator. 

Table \ref{tab:svmclass} shows the resulting classification errors on the test set (obtained for the SVM trained on the full training set). 
\begin{table}
\centering
\setlength{\tabcolsep}{3pt}
{\footnotesize
\begin{tabular}{ l | c c c c  || c c c c }
 & \multicolumn{4}{ c || }{Haar} & \multicolumn{4}{c }{RBIO2.2} \\
 & abs & ReLU & tanh & LogSig & abs & ReLU & tanh & LogSig \\ \hline
n.p. & 0.55 & 0.57 & 1.41 & 1.49   & 0.50 & 0.54 & 1.01 & 1.18 \\
sub. & 0.60 & 0.58 & 1.25 & 1.45   & 0.59 & 0.62 & 1.04 & 1.13 \\
max. & 0.61 & 0.60 & 0.68 & 0.76    & 0.55 & 0.56 & 0.71 & 0.75 \\
avg. & 0.57 & 0.58 & 1.26 & 1.44   & 0.51 & 0.60 & 1.04 & 1.18 \\
\hline
\end{tabular}}
\caption{\label{tab:svmclass} Classification error in percent for handwritten digit classification using different configurations of wavelet filters,
non-linearities, and pooling operators (sub.: sub-sampling; max.: max-pooling; avg.: average-pooling; n.p.: no pooling).}
\end{table}
Configurations employing RBIO2.2 wavelets tend to yield a marginally lower classification error than those using Haar wavelets. For the tanh and LogSig non-linearities, max-pooling leads to a considerably lower classification error than other pooling operators. The configurations invol\-ving the modulus and ReLU non-linearities achieve classification accuracy competitive with the state-of-the-art \cite{Bruna} (class. err.: $0.43\%$), which is based on directional non-separable wavelets with $6$ directions without intra-layer pooling. This is interesting as the separable wavelet filters employed here can be implemented more efficiently. 

\subsection{Feature importance evaluation}\label{sec:featimp}
In this experiment, we investigate the ``importance'' of the features generated by $\Phi_\Omega$ corresponding to different layers, wavelet scales, and directions in two different learning tasks, namely, facial landmark detection and handwritten digit classification. The primary goal of this experiment is to illustrate the practical relevance of the notion of local properties of $\Phi_\Omega$ as established in Section \ref{sec:deepinv}. 
For facial landmark detection we employ a RF regressor and for handwritten digit classification a RF classifier \cite{breiman2001random}. In both cases, we fix the number of trees to $30$ and select the tree depth using out-of-bag error estimates (noting that increasing the number of trees does not significantly increase the accuracy). The impurity measure used for learning the node tests is the mean square error for facial landmark detection and the Gini impurity for handwritten digit classification. In both cases, feature importance is assessed using the Gini importance \cite{breiman1984classification}, 
averaged over all trees. 
The Gini importance $I(\theta,T)$ of feature $\theta$ in the (trained) tree $T$ is defined as $I(\theta,T) = \sum_{\ell \in T \colon \varphi(\ell) = \theta} \frac{n_\ell}{n_\mathrm{tot}} (\hat \imath_\ell - \frac{n_{\ell_L}}{n_\ell} \hat \imath_{\ell_L} - \frac{n_{\ell_R}}{n_\ell} \hat \imath_{\ell_R})$, where $\varphi(\ell)$ denotes the feature determined in the training phase for the test at node $\ell$, $n_\ell$ is the number of training samples passed through node $\ell$, $n_\mathrm{tot} = \sum_{\ell \in T} n_\ell$, $\hat \imath_\ell$ is the impurity at node $\ell$, and $\ell_L$ and $\ell_R$ denote the left and right child node, respectively, of node $\ell$. For the feature extractor $\Phi_\Omega$ we set $D = 4$, employ Haar wavelets with $J=3$ scales and the modulus non-linearity in every network layer, no pooling in the first layer and average pooling with uniform weights $1/S^2_d$, $S_d = 2$, in layers $d=2,3$.
\vspace{-0.3cm}
\paragraph{Facial landmark detection.} We use the Caltech 10,000 Web Faces data base \cite{angelova2005pruning}. Each of the $7092$ images in the data base depicts one or more faces in different contexts (e.g., portrait images, groups of people). The data base contains annotations of the positions of eyes, nose, and mouth for at least one face per image. The lear\-ning task is to estimate the positions of these facial landmarks. The annotations serve as ground truth for training and tes\-ting. We preprocess the data set as follows. The patches containing the faces are extracted from the images using the Viola-Jones face detector \cite{viola2004robust}. After discarding false positives, the patches are converted to grayscale and resampled to size $120 \times 120$ (using linear interpolation), before feeding them to the feature extractor $\Phi_\Omega$. This procedure yields a dataset containing a total of $8776$ face images. We select 80\% of the images uniformly at random to form a training set and use the remaining images for testing. We train a separate RF for each facial landmark. Following \cite{dantone2012real} we report the localization error, i.e., the $\ell_2$-distance between the estimated and the ground truth landmark positions, on the test set as a fraction of the (true) inter-ocular distance. The errors obtained are: left eye: 0.062; right eye: 0.064; nose; 0.080, mouth: 0.095. As an aside, we note that these values are comparable with the ones reported in \cite{dantone2012real} for a conditional RF using patch comparison features (evaluated on a different dataset and a larger set of facial landmarks).
\vspace{0.3cm}
\vspace{-.6cm}
\paragraph{Handwritten digit classification.} For this experiment, we again rely on the MNIST dataset. The training set is obtained by sampling uniformly at random $1,000$ images per digit from the MNIST training dataset and we use the complete MNIST test set. We train two RFs, one based on unmodified images, and the other one based on images subject to a random uniform displacement of at most $4$ pixels in (positive and negative) $x$ and $y$ direction to study the impact of offsets on feature importance. The resulting RFs achieve a classification error of 4.2\% and 9.6\%, respectively.
\vspace{-0.3cm}
\paragraph{Discussion.} Figure \ref{fig:featimpfaces} shows the cumulative feature importance (per triplet of layer index, wavelet scale, and direction, ave\-raged over all trees in the respective RF) in handwritten digit classification and in facial landmark detection. Table~\ref{tab:featimplayers} shows the corresponding cumulative feature importance for each layer.

For facial landmark detection, the features in layer $1$ clearly have the highest importance, and the feature importance decreases with increasing layer index $d$. For handwritten digit classification using the unshifted MNIST images, the cumulative importance of the features in the second/third layer relative to those in the first layer is considerably higher than in facial landmark detection (see Table \ref{tab:featimplayers}). For the translated MNIST images, the importance of the features in the second/third layer is significantly higher than those in the $0$-th and in the first layer. An explanation for this observation could be as follows: In a classification task small sensitivity to translations is beneficial. Now, according to our theory (see Section \ref{sec:deepinv}) translation sensitivity, indeed, decreases with increasing layer index for average pooling as used here. For localization of landmarks, on the other hand, the RF needs features that are covariant on the fine grid of the input image thus favoring features in the layers closer to the root.


\newcommand{\cola}{red!50!white}
\newcommand{\colb}{green!50!white}
\newcommand{\colc}{blue!50!white}
\newcommand{\errbwidth}{1}
\newcommand{\bwidth}{2.7}

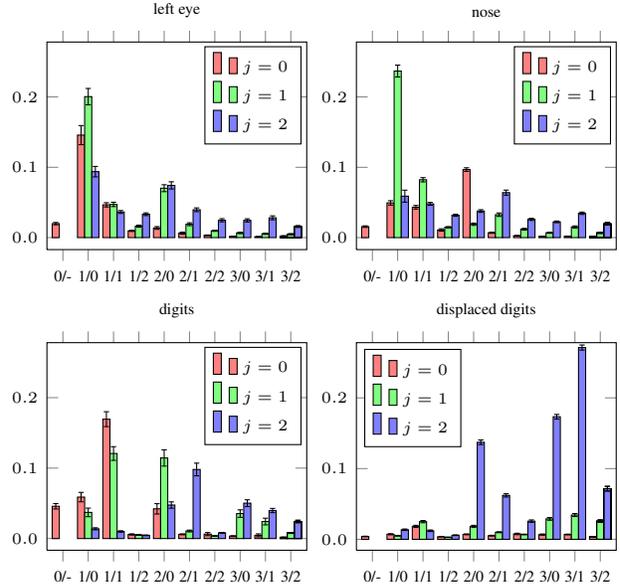
\begin{figure}
\centering
\begin{tikzpicture}[scale=1] 
\pgfplotsset{every axis/.append style={font=\tiny},}
	\begin{groupplot}[
	group style={group size=2 by 2,horizontal sep=0.08\columnwidth, vertical sep=0.15\columnwidth},
	width=0.61\columnwidth,
	ybar=0pt,
	ymax=0.26, 
	enlargelimits=0.07,
	y tick label style={
		/pgf/number format/.cd,
		fixed,
            	fixed zerofill,
		precision=1, 
		/tikz/.cd},
	symbolic x coords={0/0,1/0,1/1,1/2,2/0,2/1,2/2,3/0,3/1,3/2,4},  
	xticklabels={0/-,1/0,1/1,1/2,2/0,2/1,2/2,3/0,3/1,3/2,4},
	xtick=data,
	] 
	\nextgroupplot[title = left eye,bar width=\bwidth]
		\addplot[fill=\cola,error bars/.cd,y dir=both,y explicit,error mark options={rotate=90,mark size=1,}]
			table[x index=0,y index=1, y error index=2, restrict x to domain=0:9] {faces_leye.dat};
		\addplot[fill=\colb,error bars/.cd,y dir=both,y explicit,error mark options={rotate=90,mark size=1,}]
			table[x index=0,y index=3, y error index=4, restrict x to domain=1:9] {faces_leye.dat};
		\addplot[fill=\colc,error bars/.cd,y dir=both,y explicit,error mark options={rotate=90,mark size=1,}]
			table[x index=0,y index=5, y error index=6, restrict x to domain=1:9] {faces_leye.dat};
		\legend{$j=0$, $j=1$, $j=2$}
	\nextgroupplot[title = nose,bar width=\bwidth]
		\addplot[fill=\cola,error bars/.cd,y dir=both,y explicit,error mark options={rotate=90,mark size=1,}]
			table[x index=0,y index=1, y error index=2, restrict x to domain=0:9] {faces_nose.dat};
		\addplot[fill=\colb,error bars/.cd,y dir=both,y explicit,error mark options={rotate=90,mark size=1,}]
			table[x index=0,y index=3, y error index=4, restrict x to domain=1:9] {faces_nose.dat};
		\addplot[fill=\colc,error bars/.cd,y dir=both,y explicit,error mark options={rotate=90,mark size=1,}]
			table[x index=0,y index=5, y error index=6, restrict x to domain=1:9] {faces_nose.dat};
		\legend{$j=0$, $j=1$, $j=2$}
	\nextgroupplot[title = digits,bar width=\bwidth]
		\addplot[fill=\cola,error bars/.cd,y dir=both,y explicit,error mark options={rotate=90,mark size=1,}] 
			table[x index=0,y index=1, y error index=2, restrict x to domain=0:9] {mnist_featimp.dat};
		\addplot[fill=\colb,error bars/.cd,y dir=both,y explicit,error mark options={rotate=90,mark size=1,}]
			table[x index=0,y index=3, y error index=4, restrict x to domain=1:9] {mnist_featimp.dat};
		\addplot[fill=\colc,error bars/.cd,y dir=both,y explicit,error mark options={rotate=90,mark size=1,}] 
			table[x index=0,y index=5, y error index=6, restrict x to domain=1:9] {mnist_featimp.dat};
		\legend{$j=0$, $j=1$, $j=2$}
	\nextgroupplot[title = displaced digits,legend pos=north west,bar width=\bwidth]
		\addplot[fill=\cola,error bars/.cd,y dir=both,y explicit,error mark options={rotate=90,mark size=1,}]
			table[x index=0,y index=1, y error index=2, restrict x to domain=0:9] {mnist_disp_featimp.dat};
		\addplot[fill=\colb,error bars/.cd,y dir=both,y explicit,error mark options={rotate=90,mark size=1,}]
			table[x index=0,y index=3, y error index=4, restrict x to domain=1:9] {mnist_disp_featimp.dat};
		\addplot[fill=\colc,error bars/.cd,y dir=both,y explicit,error mark options={rotate=90,mark size=1,}]
			table[x index=0,y index=5, y error index=6, restrict x to domain=1:9] {mnist_disp_featimp.dat};
		\legend{$j=0$, $j=1$, $j=2$}	
	\end{groupplot}
	
\end{tikzpicture}
\caption{\label{fig:featimpfaces} Average cumulative feature importance and standard error for facial landmark detection and handwritten digit classification. 
The labels on the horizontal axis indicate layer index $d$/wavelet direction (0: horizontal, 1: vertical, 2: diagonal).
}
\end{figure}

\newcommand{\tabprec}{3}

\begin{table}

\pgfplotstableread[col sep=comma]{rownames.dat}\data
\pgfplotstableread{faces_leye_cumlayers.dat}\dataleye
\pgfplotstableread{faces_reye_cumlayers.dat}\datareye
\pgfplotstableread{faces_nose_cumlayers.dat}\datanose
\pgfplotstableread{faces_mouth_cumlayers.dat}\datamouth
\pgfplotstableread{mnist_featimp_cumlayers.dat}\datamnist
\pgfplotstableread{mnist_disp_featimp_cumlayers.dat}\datamnistdisp

\pgfplotstablecreatecol[copy column from table={\dataleye}{[index] 0}] {c0} {\data}
\pgfplotstablecreatecol[copy column from table={\datareye}{[index] 0}] {c1} {\data}
\pgfplotstablecreatecol[copy column from table={\datanose}{[index] 0}] {c2} {\data}
\pgfplotstablecreatecol[copy column from table={\datamouth}{[index] 0}] {c3} {\data}
\pgfplotstablecreatecol[copy column from table={\datamnist}{[index] 0}] {c4} {\data}
\pgfplotstablecreatecol[copy column from table={\datamnistdisp}{[index] 0}] {c5} {\data}

{\footnotesize
\setlength{\tabcolsep}{3pt}
\pgfplotstabletypeset[
	/pgfplots/table/display columns/0/.style={string type,column name=}, 
	/pgfplots/table/display columns/1/.style={column name=left eye,fixed,fixed zerofill,precision=\tabprec},
	/pgfplots/table/display columns/2/.style={column name=right eye,fixed,fixed zerofill,precision=\tabprec},
	/pgfplots/table/display columns/3/.style={column name=nose,fixed,fixed zerofill,precision=\tabprec},
	/pgfplots/table/display columns/4/.style={column name=mouth,fixed,fixed zerofill,precision=\tabprec},
	/pgfplots/table/display columns/5/.style={column name=digits,fixed,fixed zerofill,precision=\tabprec},
	/pgfplots/table/display columns/6/.style={,column name=disp. digits,fixed,fixed zerofill,precision=\tabprec}, 
	every head row/.style={after row=\midrule}, 
	every last row/.style={ after row=\midrule},
	]\data}
\caption{\label{tab:featimplayers} 
Cumulative feature importance per layer. Columns 1--4: facial landmark detection. Columns 5 and 6: handwritten digit classification.
}
\end{table}


\newpage
\section*{Acknowledgments} 
The authors would like to thank C. Geiger for preliminary work on the experiments in Section \ref{sec:featimp} and M. Lerjen for help with computational issues.

\bibliography{example_paper}
\bibliographystyle{icml2016}

\clearpage

\numberwithin{equation}{section}
\numberwithin{theorem}{section}
\numberwithin{lemma}{section}
\numberwithin{proposition}{section}
\numberwithin{remark}{section}

\appendix

\section{Appendix: Additional numerical results}

\subsection{Handwritten digit classification}

For the handwritten digit classification experiment described in Section \ref{sec:digitclass}, Table \ref{tab:svmclasssup} shows the classification error for Daubechies wavelets with 2 vanishing moments (DB2).

\begin{table}[h!]
\centering
\setlength{\tabcolsep}{3pt}
{\footnotesize
\begin{tabular}{ l | c c c c }
 & \multicolumn{4}{ c  }{DB2} \\
 & abs & ReLU & tanh & LogSig \\ \hline
n.p. & 0.54 & 0.51 & 1.29 & 1.40 \\
sub. & 0.60 & 0.58 & 1.16 & 1.34 \\
max. & 0.57 & 0.57 & 0.75 & 0.67 \\
avg. & 0.52 & 0.61 & 1.16 & 1.27 \\
\hline
\end{tabular}}
\caption{\label{tab:svmclasssup} Classification errors in percent for handwritten digit classification using DB2 wavelet filters, different non-linearities, and different pooling operators (sub.: sub-sampling; max.: max-pooling; avg.: average-pooling; n.p.: no pooling).}
\end{table}

\subsection{Feature importance evaluation}

For the feature importance experiment described in Section \ref{sec:featimp}, Figure \ref{fig:featimpfaces2} shows the cumulative feature importance (per triplet of layer index, wavelet scale, and direction, ave\-raged over all trees in the respective RF) in facial landmark detection (right eye and mouth).

\newcommand{\plotsc}{1}
\renewcommand{\errbwidth}{2}
\renewcommand{\bwidth}{4.7}

\begin{figure}
\centering
\begin{tikzpicture}[scale=\plotsc] 
\pgfplotsset{every axis/.append style={font=\tiny},}
	\begin{axis}[
	title = right eye, 
	ybar=0pt,
	ymax=0.26, 
	enlargelimits=0.07, 
	y tick label style={
		/pgf/number format/.cd,
		fixed,
            	fixed zerofill,
		precision=2, 
		/tikz/.cd},
	symbolic x coords={0/0,1/0,1/1,1/2,2/0,2/1,2/2,3/0,3/1,3/2,4}, 
	xticklabels={0/-,1/0,1/1,1/2,2/0,2/1,2/2,3/0,3/1,3/2,4},
	xtick=data,
	bar width=\bwidth
	] 
		\addplot[fill=\cola,error bars/.cd,y dir=both,y explicit,error mark options={rotate=90,mark size=1,}]
			table[x index=0,y index=1, y error index=2, restrict x to domain=0:9] {faces_reye.dat};
		\addplot[fill=\colb,error bars/.cd,y dir=both,y explicit,error mark options={rotate=90,mark size=1,}]
			table[x index=0,y index=3, y error index=4, restrict x to domain=1:9] {faces_reye.dat};
		\addplot[fill=\colc,error bars/.cd,y dir=both,y explicit,error mark options={rotate=90,mark size=1,}]
			table[x index=0,y index=5, y error index=6, restrict x to domain=1:9] {faces_reye.dat};
		\legend{$j=0$, $j=1$, $j=2$}
	\end{axis}
\end{tikzpicture}\hspace{1.5cm}

\begin{tikzpicture}[scale=\plotsc] 
\pgfplotsset{every axis/.append style={font=\tiny},}
	\begin{axis}[
	title = mouth, 
	ybar=0pt,
	ymax=0.26, 
	enlargelimits=0.07, 
	y tick label style={
		/pgf/number format/.cd,
		fixed,
            	fixed zerofill,
		precision=2, 
		/tikz/.cd},
	symbolic x coords={0/0,1/0,1/1,1/2,2/0,2/1,2/2,3/0,3/1,3/2,4}, 
	xticklabels={0/-,1/0,1/1,1/2,2/0,2/1,2/2,3/0,3/1,3/2,4},
	xtick=data,
	bar width=\bwidth
	] 
		\addplot[fill=\cola,error bars/.cd,y dir=both,y explicit,error mark options={rotate=90,mark size=1,}]
			table[x index=0,y index=1, y error index=2, restrict x to domain=0:9] {faces_mouth.dat};
		\addplot[fill=\colb,error bars/.cd,y dir=both,y explicit,error mark options={rotate=90,mark size=1,}]
			table[x index=0,y index=3, y error index=4, restrict x to domain=1:9] {faces_mouth.dat};
		\addplot[fill=\colc,error bars/.cd,y dir=both,y explicit,error mark options={rotate=90,mark size=1,}]
			table[x index=0,y index=5, y error index=6, restrict x to domain=1:9] {faces_mouth.dat};
		\legend{$j=0$, $j=1$, $j=2$}
	\end{axis} 
\end{tikzpicture}
\caption{\label{fig:featimpfaces2}  Average cumulative feature importance and standard error for facial landmark detection. The labels on the horizontal axis indicate layer index $d$/wavelet direction (0: horizontal, 1: vertical, 2: diagonal).}
\end{figure}
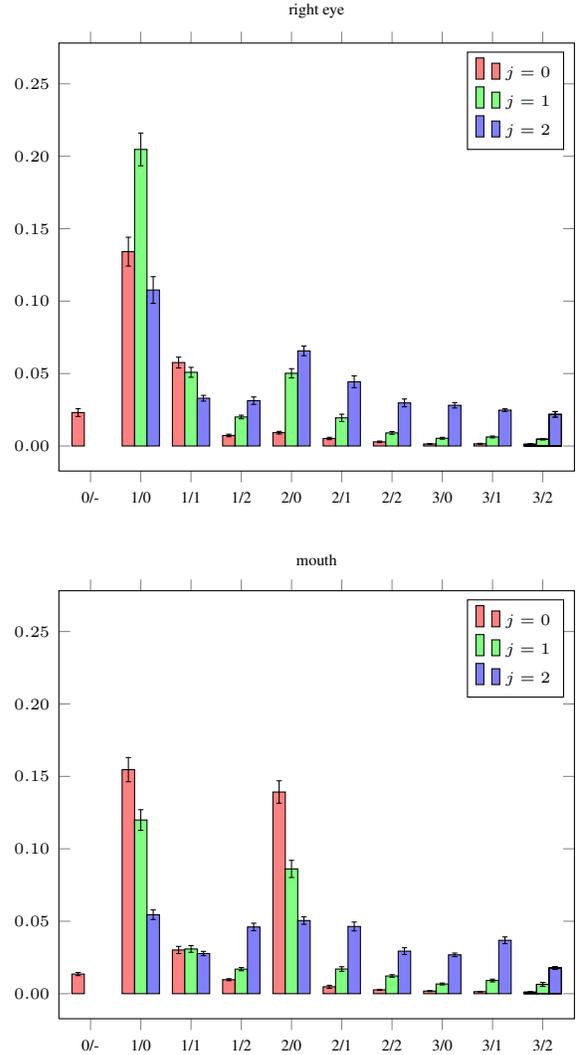

\section{Appendix: Lipschitz continuity of pooling operators}\label{app:proofLip}
We verify the Lipschitz property
$$
\|P(f)-P(h) \|_2\leq R\| f-h\|_2, \quad \forall f,h\in H_{N}, 
$$
for the pooling operators in Section \ref{sec:exmppooling}. 
\vspace{0.4cm}

\textit{Sub-sampling:} Pooling by sub-sampling is defined as
$$
P:H_N \to H_{N/S}, \quad P(f)[n]=f[Sn],\quad n \in I_{N/S},
$$
where $N/S\in \mathbb{N}$. Lipschitz continuity with $R=1$ follows from 
\begin{align*}
&\| P(f)-P(h) \|_2^2=\sum_{n\in I_{N/S}}| f[Sn] -h[Sn]|^2\\
&\leq \sum_{n\in I_N} |f[n]-h[n]|^2=\|f-h \|_2^2,\quad \forall f,h \in H_N.
\end{align*}

\textit{Averaging:} Pooling by averaging is defined as
$$
P:H_N \to H_{N/S}, \quad P(f)[n]=\sum_{k=Sn}^{Sn+S-1}\alpha_{k-Sn}f[k],
$$
for $n\in I_{N/S}$, where $N/S\in \mathbb{N}$. We start by setting $\alpha':=\max_{k\in \{0,\dots,S-1 \}}|\alpha_{k}|$. Then,  
\begin{align}
&\| P(f)-P(h) \|_2^2\nonumber\\
&=\sum_{n\in I_{N/S}}\Big|\sum_{k=Sn}^{Sn+S-1}\alpha_{k-Sn}(f[k]-h[k]) \Big|^2\nonumber\\
&\leq\sum_{n\in I_{N/S}}\Big|\sum_{k=Sn}^{Sn+S-1}\alpha'|f[k]-h[k]| \Big|^2\nonumber\\
&\leq\alpha'^2 S \sum_{n\in I_{N/S}} \sum_{k=Sn}^{Sn+S-1}\Big|f[k]-h[k] \Big|^2\label{eq:lipo1}\\
&=\alpha'^2 S \sum_{n\in I_N}\Big|f[k]-h[k] \Big|^2=\alpha'^2 S\|f-h \|^2_2\nonumber,
\end{align}
where we used $\sum_{k\in I_S}|f[k]-h[k]|\leq S^{1/2} \| f-h \|_2$, $f,h\in H_{S}$, to get \eqref{eq:lipo1}, see, e.g., \cite{Golub}.

\textit{Maximization:} Pooling by maximization is defined as
$$
P:H_N \to H_{N/S}, \quad P(f)[n]=\max_{k\in \{ Sn,\dots, Sn+S-1\}}|f[k]|,
$$
for $n\in I_{N/S}$, where $N/S\in \mathbb{N}$. We have  
\begin{align}
&\| P(f)-P(h) \|_2^2\nonumber\\
&=\sum_{n\in I_{N/S}}\big|\max_{k\in \{ Sn,\dots, Sn+S-1\}}|f[k]|\nonumber \\
&- \max_{k\in \{ Sn,\dots, Sn+S-1\}}|h[k]| \big|^2 \nonumber\\
&\leq \sum_{n\in I_{N/S}}\max_{k\in \{ Sn,\dots, Sn+S-1\}} \big|f[k]-h[k] \big|^2 \label{eq:maxpool2} \\
&\leq\sum_{n\in I_{N/S}}\sum_{k=0}^{S-1} |f[Sn+k]-h[Sn+k]|^2\label{eq:maxpool3} \\
&=\| f-h\|_2^2\nonumber,
\end{align}
where we employed the reverse triangle inequality $\big| \| f \|_\infty - \| h \|_\infty\big| \leq \| f-h\|_\infty$, $f,h\in H_{S}$, to get  \eqref{eq:maxpool2}, and in \eqref{eq:maxpool3} we used $\| f \|_\infty\leq \|f \|_2$, $f\in H_{S}$, see, e.g.,  \cite{Golub}.

\section{Appendix: Proof of Theorem \ref{mainmain}}\label{proof:defo}
We start by proving i). The key idea of the proof is---similarly to the proof of Proposition 4 in \cite{wiatowski2015mathematical}---to employ teles\-coping  series arguments. For ease of notation, we let $f_q:=U[q]f$ and $h_q:=U[q]h$, for $f,h \in H_{N_1}$,  $q\in \Lambda_1^d$. With \eqref{eq:featspacenorm} we have 
\begin{equation*}
\begin{split}
||| \Phi_\Omega(f) -\Phi_\Omega(h)|||^2&=\sum_{d=0}^{D-1}\underbrace{\sum_{q\in \Lambda_{1}^d} ||(f_q-h_q)\ast\chi_d ||^2_2}_{=:a_d}.\\ 
\end{split}
\end{equation*}
The key step is then to show that $a_d$ can be upper-bounded according to
\begin{equation}\label{eq:p1}
\begin{split}
a_d\leq b_d-b_{d+1},\hspace{0.5cm}  \,d=0, \dots,D-1,
\end{split}
\end{equation}
with $b_d:=\sum_{q\in\Lambda_1^d}\| f_q-h_q\|^2_2,$ for  $d=0, \dots,D$,
and to note that
\begin{equation*}\label{eq:p2}
\begin{split}
\sum_{d=0}^{D-1}a_d&\leq\sum_{d=0}^{D-1}(b_d-b_{d+1})= b_0 - \underbrace{b_{D}}_{\geq0}
\leq b_0\\
&= \sum_{q\in\Lambda_1^0}\| f_q-h_q\|^2_2=\| f-h\|^2_2,
\end{split}
\end{equation*}
which then yields  \eqref{eq:thm1}. Writing out \eqref{eq:p1}, it follows that we need to establish 
\begin{align}
 &\sum_{q\in\Lambda_1^d}\| (f_q-h_q)\ast\chi_d\|^2_2 \leq \sum_{q\in \Lambda_{1}^d} ||f_q-h_q\|^2_2\nonumber\\
 -&\sum_{q\in\Lambda_1^{d+1}}\| f_q - h_q\|^2_2, \hspace{0.5cm}  d=0,\dots, D-1\label{eq:p55}.
\end{align}
We start by examining the second sum on the right-hand side (RHS) in \eqref{eq:p55}. Every path 
\begin{equation*}\label{eq:ee1}
\tilde{q} \in \Lambda_{1}^{d+1}=\underbrace{\Lambda_{1}\times\dots\times\Lambda_{d}}_{=\Lambda_{1}^d}\times\Lambda_{d+1}
\end{equation*}
of length $d+1$ can be decomposed into a path $q \in \Lambda_1^{d}$ of length $d$ and an index $\lambda_{d+1} \in \Lambda_{d+1}$ according to $\tilde{q}=(q,\lambda_{d+1})$. Thanks to \eqref{aaaaa} we have  $U[\tilde{q}]=U[(q,\lambda_{d+1})]=U_{d+1}[\lambda_{d+1}]U[q]$, which yields 
\begin{align}
&\sum_{\tilde{q}\in\Lambda_1^{d+1}}\| f_{\tilde{q}}-h_{\tilde{q}}\|^2_2\nonumber=\sum_{q\in\Lambda_1^{d}}\sum_{\lambda_{d+1}\in \Lambda_{d+1}}\| U_{d+1}[\lambda_{d+1}]f_q\nonumber\\
&-U_{d+1}[\lambda_{d+1}]h_q\|^2_2\label{eq:a6}.
\end{align}
Substituting \eqref{eq:a6} into \eqref{eq:p55} and rearranging terms, we obtain 
\begin{align}
&\sum_{q\in\Lambda_1^d}\Big(\| (f_q-h_q)\ast\chi_d\|^2_2\label{eq:pp9}\\
 &+\sum_{\lambda_{d+1}\in\Lambda_{d+1}}\| U_{d+1}[\lambda_{d+1}]f_q-U_{d+1}[\lambda_{d+1}]h_q\|^2_2\Big)\label{eq:p9}\\&\leq \sum_{q\in \Lambda_{1}^d} ||f_q - h_q\|^2_2, \hspace{0.5cm}  d=0,\dots, D-1\label{eq:ppp9}.
\end{align}
We next note that the sum over the index set $\Lambda_{d+1}$ inside the brackets in  \eqref{eq:pp9}-\eqref{eq:p9} satisfies 
\begin{align}&
\sum_{\lambda_{d+1}\in\Lambda_{d+1}}\| U_{d+1}[\lambda_{d+1}]f_q-U_{d+1}[\lambda_{d+1}]h_q\|^2_2\nonumber\\
&=\sum_{\lambda_{d+1}\in\Lambda_{d+1}}\| P_{d+1}\big(\rho_{d+1}(f_q\ast g_{\lambda_{d+1}})\big)\nonumber\\
&\hspace{1.9cm}-P_{d+1}\big(\rho_{d+1}(h_q\ast g_{\lambda_{d+1}})\big)\|^2_2\nonumber\\
&\leq R^2_{d+1}\sum_{\lambda_{d+1}\in\Lambda_{d+1}}\| \rho_{d+1}(f_q\ast g_{\lambda_{d+1}})\label{eq:ieee}\\
&\hspace{1.9cm}-\rho_{d+1}(h_q\ast g_{\lambda_{d+1}})\|^2_2\label{eq:ee2}\\
&\leq R_{d+1}^2L_{d+1}^2\sum_{\lambda_{d+1}\in\Lambda_{d+1}}\| (f_q-h_q)\ast g_{\lambda_{d+1}}\|_2^2\label{eq:ieee1}, 
\end{align}
where we employed the Lipschitz continuity of  $P_{d+1}$ in \eqref{eq:ieee}-\eqref{eq:ee2} and the Lipschitz continuity of $\rho_{d+1}$ in \eqref{eq:ieee1}. Substituting the sum over the index set $\Lambda_{d+1}$ inside the brackets in  \eqref{eq:pp9}-\eqref{eq:p9} by the upper bound \eqref{eq:ieee1} yields 
\begin{align}
&\sum_{q\in\Lambda_1^d}\Big(\| (f_q-h_q)\ast\chi_d\|^2_2\nonumber\\
&+\sum_{\lambda_{d+1}\in\Lambda_{d+1}}\| U_{d+1}[\lambda_{d+1}]f_q-U_{d+1}[\lambda_{d+1}]h_q\|^2_2\Big)\nonumber\\
&\leq\sum_{q\in\Lambda_1^d}\max\{1,R_{d+1}^{2}L^2_{d+1}\}\Big(\| (f_q-h_q)\ast\chi_d\|^2_2\label{eq:ee-199}\\
&+\sum_{\lambda_{d+1}\in\Lambda_{d+1}}\| (f_q-h_q)\ast g_{\lambda_{d+1}}\|_2^2 \Big),\label{eq:e-199}
\end{align}
for $d=0, \dots,D-1$. As  $\{ g_{\lambda_{d+1}}\}_{\lambda_{d+1}\in\Lambda_{d+1}}\cup\{ \chi_d\}$ are atoms of the convolutional set $\Psi_{d+1}$, and $f_q, h_q \in H_{N_{d+1}}$, we have
\begin{align*}
&\| (f_q-h_q)\ast\chi_d\|^2_2+\sum_{\lambda_{d+1}\in\Lambda_{d+1}}\| (f_q-h_q)\ast g_{\lambda_{d+1}}\|_2^2\\
&\leq B_{d+1}\| f_q-h_q\|^2_2,
\end{align*}
which, when used in \eqref{eq:ee-199}-\eqref{eq:e-199} yields
\begin{align}
&\sum_{q\in\Lambda_1^d}\Big(\| (f_q-h_q)\ast\chi_d\|^2_2\nonumber\\
&+\sum_{\lambda_{d+1}\in\Lambda_{d+1}}\| U_{d+1}[\lambda_{d+1}]f_q-U_{d+1}[\lambda_{d+1}]h_q\|^2_2\Big)\nonumber\\
\leq&\sum_{q\in\Lambda_1^d}\max\{B_{d+1},B_{d+1}R_{d+1}^{2}L^2_{d+1}\}\| f_q-h_q\|^2_2, \label{eq:e-200}
\end{align}
for $d=0,\dots,D-1$. Finally, invoking \eqref{weak_admiss2} in \eqref{eq:e-200} we get \eqref{eq:pp9}-\eqref{eq:ppp9} and hence \eqref{eq:p1}. This completes the proof of i).

We continue with ii).  The key step in establishing \eqref{eq:thm2} is to show that for $\rho_d(0)=0$ and $P_d(0)=0$, for $d\in \{1,\dots,D-1\}$, the feature extractor $\Phi_\Omega$ satisfies $\Phi_\Omega(0)=0$, and to employ \eqref{eq:thm1} with $h=0$ which yields
$$
|||\Phi(f)|||\leq \| f\|,
$$
for $f\in H_{N_1}$. It remains to prove that $\Phi_\Omega(h)=0$ for $h=0$. For $h=0$, the operator $U_d$, $d\in \{1,2,\dots, D\}$, defined in \eqref{eq:e1} satisfies 
\begin{equation*}\label{eq:dd1aa}
(U_d[\lambda_d]h)=\underbrace{P_d\big( \underbrace{\rho_d(\underbrace{h\ast g_{\lambda_d}}_{=0})}_{=0}\big)}_{=0}, 
\end{equation*}
for $\lambda_d \in \Lambda_d$, by assumption. With the definition of $U[q]$ in \eqref{aaaaa} this then yields $(U[q]h)=0$ for $h=0$ and all $q\in \Lambda_1^d$. 
$\Phi_\Omega(0)=0$ finally follows from 
\begin{equation}\label{eq:iiii5}
\begin{split}
\Phi_\Omega(h)&=\bigcup_{d=0}^{D-1}\big\{ \underbrace{\big(U[q]h \big) \ast \chi_d}_{=0} \big\}_{q\in \Lambda_1^d}=0.
\end{split}
\end{equation}

We proceed to iii). The proof of the deformation sensitivity bound \eqref{mainmainmain} is based on two key ingredients. The first one is the Lipschitz continuity result stated in \eqref{eq:thm1}. The second ingredient, stated in Proposition \ref{prop:main} in Appendix \ref{proof:uppersuper}, is an upper bound on the deformation error $\| f-F_{\tau}f \|_2$ given by  
\begin{equation}\label{eq:p23}
\| f-F_{\tau} f\|_2 \leq 4KN_1^{1/2}\| \tau \|_{\infty}^{1/2},
\end{equation}
where $f\in C_{\text{CART}}^{N_1,K}$. We now show how \eqref{eq:thm1} and \eqref{eq:p23} can be combined to establish \eqref{mainmainmain}. To this end, we first apply \eqref{eq:thm1} with $h:=(F_{\tau}f)$ to get
\begin{equation}\label{eq:p20}
||| \Phi_\Omega(f) -\Phi_\Omega(F_{\tau} f)||| \leq \| f-F_{\tau} f \|_2,
\end{equation}
for $f   \in C_{\text{CART}}^{N_1,K} \subseteq H_{N_1}$, $N_1\in \mathbb{N}$, and $K>0$, and then replace the RHS of \eqref{eq:p20} by the RHS of \eqref{eq:p23}. This completes the proof of iii). 

\section{Appendix: Proposition \ref{prop:main}}\label{proof:uppersuper}
\begin{proposition}\label{prop:main}
For every $N\in \mathbb{N}$, every $K>0$, and every $\tau:\mathbb{R}\to [-1,1]$, we have 
 \begin{equation}\label{eq:main}ƒ
  	\|f - F_\tau f \|_2\le 4KN^{1/2}\|\tau\|_\infty^{1/2},
  \end{equation}
for all $f\in \mathcal{C}^{N,K}_{\mathrm{CART}}$.

\end{proposition}
\begin{remark}\label{rem:def}
As already mentioned at the end of Section~\ref{sec-cartoon}, excluding the interval boundary points $a,b$ in the definition of sampled cartoon functions $\mathcal{C}^{N,K}_{\mathrm{CART}}$ (see Definition~\ref{def:1d-cartoon}) is necessary for technical reasons. Specifically, without imposing this exclusion, we can not expect to get deformation sensitivity results of the form \eqref{eq:main}. This can be seen as follows. Let us assume that we seek a  bound of the form $\|  f-F_\tau f\|_2\leq C_{N,K} \| \tau \|_\infty^\alpha$, for some $C_{N,K}>0$ and some $\alpha>0$,  that applies to all $f[n]=c(n/N)$, $n\in I_N$, with $c\in \mathcal{C}^{K}_{\mathrm{CART}}$. Take $\tau(x)=1/N$, in which case the deformation $(F_\tau f)[n]=c(n/N - 1/N)$ amounts to a simple translation by $1/N$ and $\|\tau \|_\infty = 1/N\leq 1$. Let $c(x)=\mathds{1}_{[0,2/N]}(x)$. Then $c\in \mathcal{C}^{K}_{\mathrm{CART}}$ for $K=1$ and $\|  f-F_\tau f\|_2=\sqrt{2}$, which obviously does not decay with $\|\tau \|_\infty^\alpha=N^{-\alpha}$ for some $\alpha>0$. We note that this phenomenon occurs only in the discrete case.
\end{remark}

\begin{proof}
The proof of \eqref{eq:main} is based on judiciously combi\-ning deformation sensitivity bounds for the sampled  components $c_1(n/N),c_2(n/N)$, $n \in I_N$, in $(c_1+\mathds{1}_{[a,b]}c_2) \in \mathcal{C}^K_{\mathrm{CART}}$, and the sampled  indicator function $\mathds{1}_{[a,b]}(n/N)$, $n \in I_N$. The first bound, stated in Lemma \ref{lem:lip} below, reads
 \begin{equation}\label{eq:main3}
  	\|f - F_\tau f\|_2\le C N^{1/2}\|\tau\|_\infty,
  \end{equation}
and applies to discrete-time signals $f[n]=f(n/N)$, $n \in I_N$, with $f:\mathbb{R}\to \mathbb{C}$ satisfying the Lipschitz property with Lipschitz constant $C$. The second bound we need, stated in Lemma \ref{lem:ind} below,  is given by
  \begin{equation}\label{eq:main2}
  	\|\mathds{1}^N_{[a,b]} - F_\tau \mathds{1}^N_{[a,b]} \|_2 \leq 2 N^{1/2} \|\tau\|_\infty^{1/2},
  \end{equation}
  and applies to sampled indicator functions $\mathds{1}^N_{[a,b]}[n]:=\mathds{1}_{[a,b]}(n/N)$, $n \in I_N$, with $a,b \notin \{0,\frac{1}{N},\dots, \frac{N-1}{N}\}.$ We now show how \eqref{eq:main3} and \eqref{eq:main2} can be combined to establish  \eqref{eq:main}. For a sampled cartoon function $f\in \mathcal{C}^{N,K}_{\mathrm{CART}}$, i.e., 
  \begin{align*} 
 f[n]&=c_1(n/N)+\mathds{1}_{[a,b]}(n/N)c_2(n/N)\\
 &=:f_1[n]+\mathds{1}^N_{[a,b]}[n]f_2[n], \hspace{0.3cm} n\in I_N,
 \end{align*} 
we have
\begin{align}
&\|f-F_\tau f \|_2\leq \| f_1 - F_\tau f_1 \|_2 +\|\mathds{1}^N_{[a,b]}(f_2-F_\tau f_2) \|_2 \nonumber \\ &
+ \|  (\mathds{1}^N_{[a,b]} - F_\tau \mathds{1}^N_{[a,b]})(F_\tau f_2)\|_2 \label{eq:p1p1}\\
&\leq  \| f_1 - F_\tau f_1 \|_2 +\|f_2-F_\tau f_2 \|_2 \nonumber\\
&+ \| \mathds{1}^N_{[a,b]} - F_\tau \mathds{1}^N_{[a,b]}\|_2\| F_\tau f_2 \|_\infty,\nonumber
\end{align}
where in \eqref{eq:p1p1} we used 
\begin{align*}
\big(F_\tau&(\mathds{1}^N_{[a,b]}f_2)\big)[n]=(\mathds{1}_{[a,b]}c_2)(n/N-\tau(n/N))\\
&=\mathds{1}_{[a,b]}(n/N-\tau(n/N))c_2((n/N-\tau(n/N)))\\
&=(F_\tau\mathds{1}^N_{[a,b]})[n](F_\tau f_2)[n].
\end{align*}
With the upper bounds \eqref{eq:main3} and \eqref{eq:main2}, invoking properties of $\mathcal{C}^{N,K}_{\mathrm{CART}}$ (namely, (i) $c_1,c_2$ satisfy the Lipschitz pro\-perty with Lipschitz constant $C=K$ and hence $f_1[n]=c_1(n/N),f_2[n]=c_2(n/N)$, $n\in I_N$, satisfy \eqref{eq:main3} with $C=K$, and (ii) $\| F_\tau f_2 \|_\infty=\sup_{n\in I_N}|(F_\tau f_2)[n]| =\sup_{n\in I_N}|c_2(n/N-\tau(n/N))|\leq \sup_{x\in \mathbb{R}}|c_2(x)|=\| c_2\|_\infty\leq K)$, this yields 
\begin{align*}
\|f-F_\tau f \|_2 \leq&  \, 2KN^{1/2}\, \| \tau \|_\infty +  2K  N^{1/2} \|\tau\|_\infty^{1/2}\label{eq:p1p}\\
\leq&\, 4KN^{1/2} \| \tau \|_\infty^{1/2} \nonumber,
\end{align*}
where in the last step we used $\| \tau \|_\infty\leq \| \tau \|_\infty^{1/2}$, which is thanks to the assumption $\| \tau \|_\infty \leq 1$. This completes the proof of \eqref{eq:main}. 
\end{proof}
It remains to establish \eqref{eq:main3} and \eqref{eq:main2}.
\begin{lemma}\label{lem:lip}
Let  $c:\mathbb{R}\to \mathbb{C}$ be Lipschitz-continuous with Lipschitz constant $C$. Let further $f[n]:=c(n/N)$, $n \in I_N$. Then, 
	\begin{equation*}\label{eq:pp1}
		\|f - F_\tau f\|_2 \le CN^{1/2}  \|\tau\|_\infty.
	\end{equation*}
\end{lemma}
\begin{proof}
Invoking the Lipschitz property of $c$ accor\-ding to
\begin{align*}
		\|f - F_\tau f \|_2^2&=\sum_{n\in I_N} |f[n] - (F_\tau f)[n]|^2\\&=\sum_{n\in I_N} |c(n/N) -  c(n/N-\tau(n/N))|^2\\ &\leq C^2 \sum_{n\in I_N}|\tau(n/N)|^2\leq C^2N \|\tau \|_\infty^2
\end{align*}
completes the proof.
\end{proof}
We continue with a deformation sensitivity result for sampled indicator functions $\mathds{1}_{[a,b]}(x)$.
\begin{lemma}\label{lem:ind}
Let $[a,b]\subseteq [0,1]$ and set $\mathds{1}^N_{[a,b]}[n]:=\mathds{1}_{[a,b]}(n/N)$, $n \in I_N$, with $a,b \notin \{0,\frac{1}{N},\dots, \frac{N-1}{N}\}.$ Then, we have
 \begin{equation*}\label{eq:cartdef}
 	\|\mathds{1}^N_{[a,b]} - F_\tau \mathds{1}^N_{[a,b]} \|_2
 	\le 2N^{1/2} \|\tau\|_\infty^{1/2}.
 \end{equation*}
\end{lemma}
\begin{proof}
In order to upper-bound 
\begin{align*}
		\|\mathds{1}^N_{[a,b]} &- F_\tau \mathds{1}^N_{[a,b]} \|_2^2=\sum_{n\in I_N} |\mathds{1}^N_{[a,b]}[n] - (F_\tau \mathds{1}^N_{[a,b]})[n]|^2\\&=\sum_{n\in I_N} |\mathds{1}_{[a,b]}(n/N) -  \mathds{1}_{[a,b]}(n/N-\tau(n/N))|^2,
\end{align*}
we first note that the summand $h(n):=|\mathds{1}_{[a,b]}(n/N) -  \mathds{1}_{[a,b]}(n/N-\tau(n/N))|^2$ satisfies $h(n)=1$, for $n\in S$, where 
\begin{align*}
S&:=\Big\{ n\in I_N \, \Big|\, \frac{n}{N}\in [a,b] \  \text{ and }\  \frac{n}{N}-\tau\Big(\frac{n}{N}\Big)\notin [a,b] \Big\}\\
& \ \, \cup \Big\{ n\in I_N \, \Big|\, \frac{n}{N}\notin [a,b] \  \text{ and }\  \frac{n}{N}-\tau\Big(\frac{n}{N}\Big)\in [a,b] \Big\},
\end{align*}
and $h(n)=0$, for $n\in I_N\backslash S$. Thanks to $a,b \notin \{0,\frac{1}{N},\dots,\frac{N-1}{N}\}$, we have $S\subseteq \Sigma$, where
\begin{align*}
\Sigma&:=\Big\{ n\in \mathbb{Z} \, \Big|\, \Big| \frac{n}{N} - a \Big|< \| \tau \|_\infty \Big\}\\
& \ \, \cup\Big\{ n\in \mathbb{Z} \, \Big|\, \Big| \frac{n}{N} - b \Big|< \| \tau \|_\infty \Big\}.
\end{align*}
The cardinality of the set $\Sigma$ can be upper-bounded by  $2\, \frac{2\| \tau \|_\infty}{1/N}$, which then yields
\begin{align}
&\|\mathds{1}^N_{[a,b]} - F_\tau \mathds{1}^N_{[a,b]} \|_2^2=\sum_{n\in I_N}|h(n)|^2\nonumber\\
&=\sum_{n\in S}1 \leq \sum_{n\in \Sigma} 1\leq 4 N \| \tau \|_\infty.\label{eq:cookok}
\end{align}
This completes the proof.

\begin{remark}
For general $a,b \in [0,1]$, i.e., when we drop the assumption $a,b \notin \{0,\frac{1}{N},\dots,\frac{N-1}{N}\}$, it follows that $S\subseteq \Sigma'$, where
\begin{align*}
\Sigma'&:=\Big\{ n\in \mathbb{Z} \, \Big|\, \Big| \frac{n}{N} - a \Big|\leq \| \tau \|_\infty \Big\}\\
& \ \, \cup\Big\{ n\in \mathbb{Z} \, \Big|\, \Big| \frac{n}{N} - b \Big|\leq \| \tau \|_\infty \Big\}.
\end{align*}
Noting that the cardinality of $\Sigma'$ can be upper-bounded by $2\, \big(\frac{2\| \tau \|_\infty}{1/N}+1\big)=4N\| \tau \|_\infty +2$, this then yields (similarly to \eqref{eq:cookok})
\begin{align*}
&\|\mathds{1}^N_{[a,b]} - F_\tau \mathds{1}^N_{[a,b]} \|_2^2 \leq \sum_{n\in \Sigma} 1\leq 4 N \| \tau \|_\infty +2,
\end{align*}
which shows that the deformation error---for general $a,b\in [0,1]$---does not decay with $\| \tau \|^\alpha_\infty$ for some $\alpha>0$ (see also the example in Remark \ref{rem:def}).
\end{remark}
\vspace{-0.6cm}
\end{proof}

\section{Appendix:  Theorem \ref{thm:deep}}\label{proof:nonexpan3}
We start by establishing i). For ease of notation, again, we let $f_q:=U[q]f$ and $h_q:=U[q]h$, for $f,h \in H_{N_1}$,  $q\in \Lambda_1^d$. We have
\begin{align}
||| \Phi^d_\Omega(f) &-\Phi^d_\Omega(h)|||^2=\sum_{q\in \Lambda_{1}^d} ||(f_q-h_q)\ast\chi_d ||^2_2\label{eq:foot}\\ 
&\leq \| \chi_d \|_1^2 \underbrace{\sum_{q\in \Lambda_{1}^d} ||(f_q-h_q) ||^2_2}_{=:a_d}\label{abc2},
\end{align}
where \eqref{abc2} follows by Young's inequality \cite{folland2015course}. 
\begin{remark}\label{rem_proof}
We emphasize that \eqref{eq:foot} can also be upper-bounded by $B_{d+1}\sum_{q\in \Lambda_{1}^d} ||(f_q-h_q) ||^2_2$, which follows from the fact that $\{ g_{\lambda_{d+1}}\}_{\lambda_{d+1}\in \Lambda_{d+1}}\cup \{ \chi_d \}$ are atoms of the convolutional set $\Psi_{d+1}$ with Bessel bound $B_{d+1}$. Hence, one can substitute $\| \chi_d\|_1$ in \eqref{lipo21} by $B_{d+1}$. 
\end{remark}

The key step is then to show that $a_d$ can be upper-bounded according to
\begin{equation}\label{eq:p1p}
\begin{split}
a_{k}\leq (B_kL_k^2R_k^2)a_{k-1},\hspace{0.5cm}  k=1, \dots,d,
\end{split}
\end{equation}
and to note that
\begin{equation*}\label{eq:p2}
\begin{split}
&a_d\leq  (B_dL_d^2R_d^2)a_{d-1} \leq \dots \leq  \Big(\prod_{k=1}^d B_kL_k^2R_k^2\Big)a_{0}\\
&= \Big(\prod_{k=1}^d B_kL_k^2R_k^2\Big) \sum_{q\in\Lambda_1^0}\| f_q-h_q\|^2_2\\
&=\Big(\prod_{k=1}^d B_kL_k^2R_k^2\Big)\| f-h\|^2_2,
\end{split}
\end{equation*}
which yields  \eqref{eq:thmb1}. We now establish \eqref{eq:p1p}. Every path 
\begin{equation*}\label{eq:ee1}
\tilde{q} \in \Lambda_{1}^{k}=\underbrace{\Lambda_{1}\times\dots\times\Lambda_{k-1}}_{=\Lambda_{1}^{k-1}}\times\Lambda_{k}
\end{equation*}
of length $k$ can be decomposed into a path $q \in \Lambda_1^{k-1}$ of length $k-1$ and an index $\lambda_{k} \in \Lambda_{k}$ according to $\tilde{q}=(q,\lambda_{k})$. Thanks to \eqref{aaaaa} we have  $U[\tilde{q}]=U[(q,\lambda_{k})]=U_{k}[\lambda_{k}]U[q]$, which yields 
\begin{align}
&\sum_{\tilde{q}\in\Lambda_1^{k}}\| f_{\tilde{q}}-h_{\tilde{q}}\|^2_2\nonumber=\sum_{q\in\Lambda_1^{k-1}}\sum_{\lambda_{k}\in \Lambda_{k}}\| U_{k}[\lambda_{k}]f_q\nonumber\\
&-U_{k}[\lambda_{k}]h_q\|^2_2\label{eq:6}.
\end{align}
We next note that the term inside the sums on the RHS in \eqref{eq:6} satisfies 
\begin{align}&\| U_{k}[\lambda_{k}]f_q-U_{k}[\lambda_{k}]h_q\|^2_2\nonumber\\
&=\| P_{k}\big(\rho_{k}(f_q\ast g_{\lambda_{k}})\big)-P_{k}\big(\rho_{k}(h_q\ast g_{\lambda_{k}})\big)\|^2_2\nonumber\\
&\leq L_k^2R_k^2  \| (f_q-h_q)\ast g_{\lambda_{k}}\|_2^2 \label{abc6},
\end{align}
where we used the Lipschitz continuity of $P_k$ and $\rho_k$ with Lipschitz constants $R_k>0$ and $L_k>0$, respectively. As $\{ g_{\lambda_{k}}\}_{\lambda_{k}\in\Lambda_{k}}\cup\{ \chi_{k-1}\}$ are the atoms of the convolutional set $\Psi_{k}$, and $f_q, h_q \in H_{N_{k}}$ by \eqref{aaaaa}, we have
\begin{align*}
&\sum_{\lambda_{k}\in\Lambda_{k}}\| (f_q-h_q)\ast g_{\lambda_{k}}\|_2^2\leq B_{k}\| f_q-h_q\|^2_2,
\end{align*}
which, when used in \eqref{abc6} together with \eqref{eq:6}, yields 
\begin{align*}
&\sum_{\tilde{q}\in\Lambda_1^{k}}\| f_{\tilde{q}}-h_{\tilde{q}}\|^2_2\leq B_kL_k^2R_k^2 \sum_{q\in \Lambda_1^{k-1}} \| f_q-h_q\|_2^2,\label{abc3}
\end{align*}
and hence establishes \eqref{eq:p1p}, thereby completing the proof of i). 

We now turn to ii). The proof of \eqref{eq:thmb2} follows---as in the proof of ii) in Theorem \ref{mainmain} in Appendix \ref{proof:defo}---from \eqref{eq:thmb1} together with  $\Phi_\Omega^d(h)=\{(U[q]h)\ast \chi_d \}_{q\in \Lambda_1^d}=0$ for $h=0$, see \eqref{eq:iiii5}.

We continue with iii). The proof of the deformation sensitivity bound \eqref{abc1a} is based on two key ingredients. The first one is the Lipschitz continuity result in \eqref{eq:thmb1}. The second ingredient is, again, the deformation sensitivity bound \eqref{eq:main} stated in Proposition \ref{prop:main} in Appendix \ref{proof:uppersuper}. Combining \eqref{eq:thmb1} and \eqref{eq:main}---as in the proof of iii) in Theorem \ref{mainmain} in Appendix \ref{proof:defo}---then establishes \eqref{abc1a} and completes the proof of iii). 

We proceed to iv). For ease of notation, again, we let $f_q:=U[q]f$, for $f \in H_{N_1}$, $q\in \Lambda_1^d$. Thanks to \eqref{aaaaa}, we have $f_q \in H_{N_{d+1}}$, for $q\in \Lambda_1^d$. The key step in establishing \eqref{eq:deepo} is to show that the operator $U_k$, $k\in \{1,2,\dots, d\}$, defined in \eqref{eq:e1} satisfies the relation
\begin{equation}\label{eq:dd1}
(U_k[\lambda_k]T_mf)=T_{m/{S_k}}(U_{k}[\lambda_k]f), 
\end{equation}
for $f\in H_{N_k}$, $m\in \mathbb{Z}$ with $\frac{m}{S_k}\in \mathbb{Z}$, and $\lambda_k \in \Lambda_k$.
With the definition of $U[q]$ in \eqref{aaaaa} this then yields 
\begin{equation}\label{eq:i101}
\begin{split}
(U[q]T_mf)=T_{m/(S_1\cdots S_{d})}(U[q]f), 
\end{split}
\end{equation}
for  $f\in H_{N_1}$, $m\in \mathbb{Z}$ with $\frac{m}{S_1\dots S_d}\in \mathbb{Z}$, and $q \in \Lambda_1^d$.
The identity \eqref{eq:deepo} is then a direct consequence of \eqref{eq:i101} and the translation-covariance of the circular convolution operator (which holds thanks to $\frac{m}{S_1\dots S_d}\in \mathbb{Z}$): 
\begin{equation*}\label{eq:i5}
\begin{split}
\Phi_\Omega^d(T_mf)&=\big\{ \big(U[q]T_mf \big) \ast \chi_d \big\}_{q\in \Lambda_1^d}\\
&=\big\{ \big(T_{m/(S_1\cdots S_d)}U[q]f\big) \ast \chi_d \big\}_{q\in \Lambda_1^d}\\
&=\big\{ T_{m/(S_1\cdots S_d)}\big((U[q]f) \ast \chi_d\big) \big\}_{q\in \Lambda_1^d}\\
&=T_{m/(S_1\cdots S_d)}\Phi_\Omega^d(f),
\end{split}
\end{equation*}
for $f\in H_{N_1}$ and $m\in \mathbb{Z}$ with $\frac{m}{S_1\dots S_d}\in \mathbb{Z}$. 
It remains to establish \eqref{eq:dd1}:
\begin{align}
(U_k[\lambda_k]T_mf)&=\Big(P_k\big(\rho_k((T_mf)\ast g_{\lambda_k})\big)\Big)\nonumber\\
&=\Big(P_k\big(\rho_k(T_m(f\ast g_{\lambda_k}))\big)\Big)\label{eq:d1}\\
&=\Big(P_k\big(T_m(\rho_k(f\ast g_{\lambda_k}))\big)\Big)\label{eq:d2},
\end{align} where in \eqref{eq:d1} we used  the translation covariance of the circular convolution operator (which holds thanks to $m \in \mathbb{Z}$), and in \eqref{eq:d2} we used the fact that point-wise non-linearities commute with the translation operator  thanks to
\begin{align*}
(\rho_kT_mf)[n]&=\rho_k((T_mf)[n])\\
&=\rho_k(f[n-m])=(T_m\rho_kf)[n],
\end{align*}
for $f\in H_{N_k}$, $n\in I_{N_k}$, and $m\in \mathbb{Z}$. Next, we note that the pooling operators $P_k$ in Section \ref{sec:exmppooling} (namely, sub-sampling, average pooling, and max-pooling) can all be written as $(P_kf)[n]=(P'_kf)[S_kn]$, for some $P'_k$ that commutes with the translation operator, namely, for (i) sub-sampling   $(P'_kf)[n]=f[n]$, with $(P'_kT_mf)[n]=(T_mf)[n]=f[n-m]=(T_mP'_kf)[n]$, (ii) average pooling $(P'_kf)[n]= \sum_{l=n}^{n+S_k-1}\alpha_{l-n} f[l]$ with 
\begin{align*}
(P'_kT_mf)[n]&=\sum_{l=n}^{n+S_k-1}\alpha_{l-n}f[l-m]\\
&=\sum_{l'=(n-m)}^{(n-m)+S_k-1}\alpha_{l-(n-m)}f[l']\\
&=(T_mP'_kf)[n],
\end{align*}
and for (iii) max-pooling  $(P'_kf)[n]=\max_{l\in \{n,\dots,n+S_k-1 \}}|f[l]|$ with 
 \begin{align*}
(P'_kT_mf)[n]&=\max_{l\in \{n,\dots,n+S_k-1 \}}|f[l-m]|\\
&=\max_{(l-m)\in \{n-m,\dots,(n-m)+S_k-1 \}}|f[l-m]|\\
&=\max_{l'\in \{(n-m),\dots,(n-m)+S_k-1 \}}|f[l']|\\
&=(T_mP'_kf)[n],
\end{align*}
in all three cases for $f\in H_{N_k}$, $n\in I_{N_k}$, and $m\in \mathbb{Z}$. This then yields
\begin{align}
(P_kT_mf)[n]&=(P'_kT_mf)[S_kn]=(T_mP'_kf)[S_kn]\nonumber\\
&=P'_k(f)[S_kn-m]\nonumber\\
&=P'_k(f)[S_k(n-S_k^{-1}m)]\nonumber\\
&=P_k(f)[n-S_k^{-1}m]\nonumber\\
&=(T_{m/S_k}P_kf)[n]\label{eq:ddd1},
\end{align}
for $f\in H_{N_k}$ and $n\in I_{N_{k+1}}$. Here, we used $m/S_k \in \mathbb{Z}$, which is by assumption. Substituting \eqref{eq:ddd1} into \eqref{eq:d2} finally yields  
$$
(U_k[\lambda_k]T_mf)=T_{m/S_k}U_k[\lambda_k]f,
$$
for $f\in H_{N_k}$, $m\in \mathbb{Z}$ with $\frac{m}{S_k}\in \mathbb{Z}$, and $\lambda_k \in \Lambda_k$. This completes the proof of \eqref{eq:dd1} and hence establishes \eqref{eq:deepo}.

\end{document}